\documentclass[journal]{IEEEtran}
\renewcommand{\baselinestretch}{1.0}
\renewcommand{\theenumii}{\theenumi.\arabic{enumii}}

\usepackage[utf8]{inputenc}
\usepackage{subfigure}
\usepackage{amssymb}
\usepackage{color,soul}
\usepackage[T1]{fontenc}
\usepackage{multicol}
\usepackage{multirow}
\usepackage{textcomp}
\usepackage{multirow}
\usepackage{bigdelim}
\usepackage{graphicx}
\usepackage{epstopdf}
\usepackage{amssymb}
\usepackage{amsmath}
\usepackage{psfrag}
\usepackage{amsthm}
\usepackage{subfigure}
\usepackage{algorithmic,algorithm}
\usepackage{mathrsfs}
\usepackage{framed}
\usepackage{color}
\usepackage{multirow,enumerate}
\usepackage{bigdelim}
\definecolor{shadecolor}{rgb}{0.92,0.92,0.92}
\usepackage{soul}
\usepackage{color}
\usepackage{cuted}
\usepackage{cite}
\usepackage{etoolbox}
\makeatletter 
\pretocmd\@bibitem{\color{black}\csname keycolor#1\endcsname}{}{\fail}
\newcommand\citecolor[1]{\@namedef{keycolor#1}{\color{blue}}}
\makeatother

\usepackage{hyperref}
\usepackage{array}
\usepackage{subfigure}

\theoremstyle{definition}

\newtheorem{remark}{Remark}

\newtheorem{lemma}{Lemma}
\newtheorem{corollary}{Corollary}
\newtheorem{theorem}{Theorem}

\makeatletter
\newcommand{\vast}{\bBigg@{3.2}}
\newcommand{\Vast}{\bBigg@{4.5}}

\makeatother

\graphicspath{{figs}}

\begin{document}

\title{Velocity Obstacle Based Risk-Bounded Motion Planning for Stochastic Multi-Agent Systems}

\author{
		Xiaoxue Zhang,
		Jun Ma,
		Zilong Cheng, 	
		Masayoshi Tomizuka,~\textit{Life Fellow, IEEE,}
		and Tong Heng Lee	
		\thanks{X. Zhang, Z. Cheng, and T. H. Lee are with the NUS Graduate School for Integrative Sciences and Engineering, National University of Singapore, Singapore 119077 (e-mail: xiaoxuezhang@u.nus.edu; zilongcheng@u.nus.edu; eleleeth@nus.edu.sg).}
	\thanks{Jun Ma is with the Robotics and Autonomous Systems Thrust and the Department of Electronic and Computer Engineering, The Hong Kong University of Science and Technology, Clear Water Bay, Kowloon, Hong Kong SAR, China (e-mail: jun.ma@ust.hk).}
		\thanks{M. Tomizuka is with the Department of Mechanical Engineering, University of California, Berkeley, CA 94720 USA (e-mail: tomizuka@berkeley.edu).}
		\thanks{This work has been submitted to the IEEE for possible publication. Copyright may be transferred without notice, after which this version may no longer be accessible.}
}
\maketitle

\begin{abstract}
In this paper, we present an innovative risk-bounded motion planning methodology for stochastic multi-agent systems. For this methodology, the disturbance, noise, and model uncertainty are considered; and a velocity obstacle method is utilized to formulate the collision-avoidance constraints in the velocity space. With the exploitation of geometric information of static obstacles and velocity obstacles, a distributed optimization problem with probabilistic chance constraints is formulated for the stochastic multi-agent system. Consequently, collision-free trajectories are generated under a prescribed collision risk bound. Due to the existence of probabilistic and disjunctive constraints, the distributed chance-constrained optimization problem is reformulated as a mixed-integer program by introducing the binary variable to improve computational efficiency. This approach thus renders it possible to execute the motion planning task in the velocity space instead of the position space, which leads to smoother collision-free trajectories for multi-agent systems and higher computational efficiency. Moreover, the risk of potential collisions is bounded with this robust motion planning methodology. To validate the effectiveness of the methodology, different scenarios for multiple agents are investigated, and the simulation results clearly show that the proposed approach can generate high-quality trajectories under a predefined collision risk bound and avoid potential collisions effectively in the velocity space.
\end{abstract}


\section{Introduction}
Motion planning is one of the essential components of intelligent robots, and a high-quality trajectory can effectively improve robots' intelligence and autonomy~\cite{semnani2020force,hang2022decision,zhang2020trajectory,hang2021cooperative}. 
Motion planning in multi-agent systems aims to generate feasible trajectories for all agents from their initial states to the desired goal states without any collision with each other and obstacles. In this process, various constraints (including the collision-avoidance constraints and system dynamics constraints) are considered~\cite{lyu2019multivehicle,zhang2021sequential,duan2021adaptive,pan2021multilayer,bono2021swarm}.

The velocity obstacle method is an effective decentralized motion planning method for multiple agents, which defines geometric conic regions of the feasible velocities for agents. The concepts of collision cone and velocity obstacle are firstly introduced in~\cite{fiorini1998motion}. Essentially, given the collision scenario geometrically, this method finds the range of feasible velocity quickly with minimal obstacle information without any prior knowledge or prediction~\cite{douthwaite2019velocity}. There are several variants of the velocity obstacle approach, such as reciprocal velocity obstacle~\cite{van2011reciprocal}, generalized velocity obstacle~\cite{wilkie2009generalized}, hybrid reciprocal velocity obstacle~\cite{snape2011hybrid}, etc. Most of these variants require only local information about the environment and are less computationally demanding. However, the velocity obstacle based methods consider the feasible velocity direction with neglecting the optimality and quality of the planned trajectory. Besides, these methods cannot address the system dynamics with higher complexity (for example, a double-integrator dynamics); and this becomes a burden in practical applications.

In order to find the feasible velocity based on the static obstacles and collision regions provided by the velocity obstacle method, the model predictive control (MPC)
can be utilized~\cite{zhang2021semi}. The MPC-based methods have the capability of addressing various constraints as part of the control synthesis problem~\cite{ zhang2019integrated, huang2021dynamic}.
In~\cite{katriniok2019nonlinear}, a distributed motion planning scheme for multiple automated vehicles is presented by solving a nonlinear MPC problem. \cite{hegde2016multi} proposes a motion planning method for multiple agents under the presence of arbitrary polygonal obstacles by using the vector field approach. In~\cite{huang2021dynamic}, a decentralized MPC method is utilized for online trajectory planning of multiple driverless vehicles.
However, these MPC-based collision-avoidance methods need to address the nonconvex collision-avoidance constraints, which lead to significant computational burden. Besides, most prior research on multi-agent motion planning methods focus on deterministic systems. In practice, due to the inevitable existence of the disturbance from sensors, controllers, and actuators, as well as the uncertainty from system dynamics, these deterministic motion planning methods may suffer from the limited collision-avoidance performance.

In this sense, developing a robust motion planning methodology is meaningful for  stochastic multi-agent systems, such that the agents exhibits good robustness properties in the presence of other moving agents and static obstacles.
One effective method is to characterize the uncertainties in a probabilistic manner and find the optimal sequence of control inputs subject to chance constraints~\cite{zhu2019chance,peng2021seperated}. Thus, a risk bound can be stipulated in the chance constraints.
There are several previous works regarding the chance-constrained motion planning with the existence of obstacles~\cite{zhang2020trajectory, blackmore2011chance}. However, most of the existing probabilistic approaches still attempt the motion planning task in the position space, which means the positional information is mainly utilized to make a collision-avoidance decision. Indeed, the consideration of integration in the problem formulation promotes the smoothness of the trajectory. As we consider the potential collisions with other agents in the velocity space, the system inputs are obtained with the integration operator of the velocity vector, which produces a high-quality trajectory.

Motivated by these results, this paper presents a novel approach to generate risk-bounded and optimal trajectories for stochastic multi-agent systems in the velocity space. The feasible regions of velocity vector for the ego agent can be obtained by using the geometric information of the velocity obstacles concerning other agents. Then, the feasible regions are formulated as probabilistic collision-avoidance constraints, and further transformed into deterministic mixed-integer constraints for efficient solving. The main contributions of our method are as follows. 
Firstly, due to the existence of the measurement noise, external disturbance, and modeling errors, we take the initial position uncertainty and model uncertainty into account, which significantly enhances the robustness of our proposed method for motion planning. Secondly, a collision risk can be confined within a predefined range in the proposed approach; thus, this approach provides a risk-bounded optimal trajectory for a stochastic multi-agent system. By specifying the risk of motion planning, the desired level of conservatism in the planning can be stipulated in a meaningful manner. Thirdly, this method is capable of planning trajectories in the velocity space with high computational efficiency, and much smoother trajectories are resulted by taking the integration of velocity vector in the problem formulation. Fourthly, our method also addresses the system dynamics constraint even though the system dynamics has a high complexity, and thus it makes our method more practical in many real applications compared with the velocity obstacle based methods. It is also worthwhile to mention that our method can be easily extended to complement other variants of velocity obstacle based methods, such as reciprocal velocity obstacle, hybrid reciprocal velocity obstacle, etc., depending on the particular requirements in the specific motion planning task.
 


\section{Preliminary}
\label{section:preliminary}
\subsection{Notations}
The following notations are used in the remaining text. $\mathbb R^{m\times n}$ denotes the space of real matrices with $m$ rows and $n$ columns, $\mathbb R^{n}$ means the space of $n$-dimensional real column vectors. $A^{\top}$ and $x^{\top}$ denote the transpose of the matrix $A$ and vector $x$, respectively.  
$\mathbb Z_{n}$ and $\mathbb Z_{m}^{n}$ represent the sets of positive integers $\{1,2,\cdots,n\}$ and $\{m,m+1,\cdots,n\}$ with $m<n$, respectively.  $|\mathbb N|$ represents the number of elements in the set $\mathbb N$.
The inner product or dot product of two vector $x \in \mathbb R^{n}, y \in \mathbb R^{n}$ is denoted as $x \cdot y = x^\top y$.
We use $\bigcap$ to denote logical AND and $\bigcup$ to represent logical OR. $\oplus$ denotes the Minkowski sum operation. $I^\complement$ means the complement set of the set $I$, which means $x \notin I \iff x \in I^\complement$ for any element $x$.
$\operatorname{diag}(a_1, a_2, \cdots, a_n)$ is a diagonal matrix with diagonal entries $a_1, a_2, \cdots, a_n$. 
$\|x\|$ denotes the length or magnitude of vector $x$, which is equal to the square root of the dot product of $x$, i.e., $\|x\|=\sqrt{x\cdot x}$. The operator $\|x\|_A$ for vector $x$ is defined as $\|x\|_A = x^\top A x$. 

\subsection{Velocity Obstacles}
\label{section:vo intro}
In this section, the basic concept of the velocity obstacle approach is introduced, and the geometric illustration is shown in Fig.~\ref{fig:vo_brief}. The related parameters and variables in this figure are explained below.
\begin{figure}[th]
	\centering
	\includegraphics[width=0.45\textwidth, trim=20 160 100 0,clip]{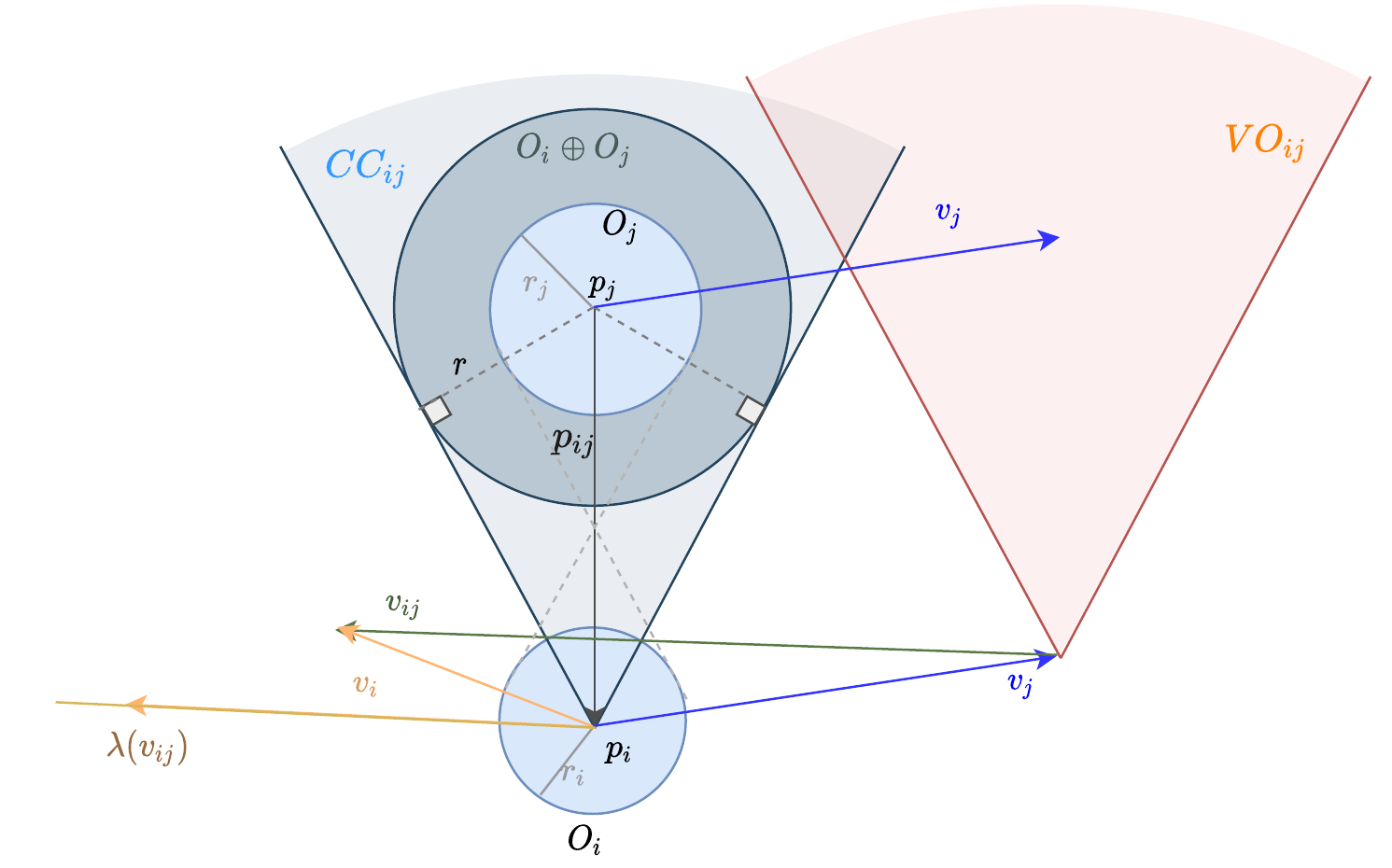}
	\caption{Illustration of the velocity obstacle}
	\label{fig:vo_brief}
\end{figure}
Let all circular moving agents $B_i$ and $B_j$ be centered at $p_i \in \mathbb R^{n}$ and $p_j \in \mathbb R^{n}$ with radius $r_i$ and $r_j$ and velocities $v_i \in \mathbb R^{n}$ and $v_j\in \mathbb R^{n}$, respectively, where 
\begin{IEEEeqnarray}{rCl}
	B_i &=& \{p_i+\mu r_i\;|\; \|\mu\| \leq 1\}\nonumber
	\\ B_j &=& \{p_j+\mu r_j\;|\; \|\mu\| \leq 1\}.
\end{IEEEeqnarray}
Set $p_{ij} = p_i - p_j$, and we have
\begin{IEEEeqnarray}{rCl}
	B_i \oplus B_j  & = & \{m_i+m_j \;|\; m_i\in B_i, m_j\in B_j\}.
\end{IEEEeqnarray}
The relative velocity between the $i$th agent and $j$th agent is $v_{ij}:=v_i-v_j$.
Let $\lambda(v_{ij})$ denote the ray with the direction $v_{ij}$ from the position $p_i$, where
\begin{IEEEeqnarray}{rCl}
	\lambda(v_{ij}) & = & \{p_i + \lambda v_{ij} \;|\; \lambda \geq 0 \}.
\end{IEEEeqnarray}
Then the agent $i$ and the agent $j$ will collide if and only if 
\begin{IEEEeqnarray}{rCl}
	\lambda(v_{ij}) \cap \left(B_i \oplus B_j\right) \neq \emptyset.
\end{IEEEeqnarray}
Therefore, a collision cone $CC_{ij}$ can be represented by
\begin{IEEEeqnarray}{rCl}
	CC_{ij} &=& \left\{ v_{ij}\;\left|\; \left(B_i \oplus B_j\right)  \cap \lambda_{ij}(v_{ij}) \neq \emptyset \right.\right\}.
\end{IEEEeqnarray}
In order to determine whether the velocity $v_i$ of the $i$th agent has a risk of collision with the $j$th agent, the velocity obstacle $VO_{ij}$ is defined as 
\begin{IEEEeqnarray}{rCl}
	VO_{ij} = \left\{ v_i \;|\; (v_i-v_j) \in \mathbb C_{ij} \right\},
\end{IEEEeqnarray}
which is equivalent to
\begin{IEEEeqnarray}{rCl}
	VO_{ij} = v_j \oplus CC_{ij},
\end{IEEEeqnarray}
for the $i$th agent. 
Any velocity $v_i \in VO_{ij}$ will result in a collision with the $j$th agent, as shown in Fig.~\ref{fig:vo_brief}, which means that the feasible velocity of the $i$th agent should be $v_i \in VO_{ij}^\complement$ to avoid the collision with the $j$th agent.

The set of all surrounding moving agents can be considered as moving obstacles for the host agent $i$. Here, we assume the set of all moving agents is $\mathbb N_n = \{1,2,\cdots, n\}$. The neighbor set of the $i$th agent is $\mathbb B_i$, which depends on the communication flow. 
For instance, the neighbors of the $i$th agent can be represented as $\mathbb B_i = \{j | j\in \mathbb N_n, j\neq i\}$, which means that the neighbors of the $i$th agent are the remaining agents except itself. 


\section{Problem Statement}
\label{section:problem_statement}

\subsection{Dynamics Model of Agents}
\label{section:model}
In this paper, we consider the case where uncertainty in the problem can be described probabilistically. We consider three sources of uncertainty:
\begin{itemize}
	\item Disturbances existing in agents are modeled as a Gaussian noise process added to the system dynamics;
	\item Since the system model is not known exactly (the uncertainty in the system model may arise due to modeling errors or linearization), we assume that model uncertainty can be modeled as a Gaussian process added to the system dynamic equation;
	\item The initial positions of agents are specified as probabilistic distributions over possible positions.  In this work, we assume that the initial positions of agents are specified as Gaussian distributions.
\end{itemize}

Therefore, a nonlinear system model is used to represent the agents' dynamics. In practice, while the low-level dynamics of the system are nonlinear, the controlled plant from the reference positions to real positions can be approximated as a low-order linear system for the purpose of motion planning. 
In this paper, we assume all agents are operating in planar space.
The dynamics of each agent $i \in \mathbb N_n$ can be represented by
\begin{IEEEeqnarray}{rCl}
\label{eq:dynamics}
	\dot x_i = f_i(x_i, u_i) + \omega_i \; \text{or}\; x_{i|k+1} = g_i\left(x_{i|k}, u_{i|k}\right) + \omega_{i|k}, \IEEEeqnarraynumspace
\end{IEEEeqnarray}
where $x_i = \begin{bmatrix} p_i^\top & v_i^\top & \cdots \end{bmatrix}^T \in \mathcal X_i \subset \mathbb R^{n_x}$ denotes the state vector consisting of positions and velocities, $u_i\in \mathcal U_i \subset \mathbb R^{n_u}$ is the control input vector, $n_x$ and $n_u$ represent the dimension of the state variable and the control input variable, respectively, and $\mathcal X_i$ and $\mathcal U_i$ are the state and control space of the $i$th agent, respectively. In the discrete-time model, the subscript $\cdot_{|k}$ means the corresponding variable at the $k$th timestamp for the agents. $f_i$ and $g_i$ are the nonlinear  continuous-time and discrete-time dynamics models of the $i$th agent. 

In this paper, we consider the Gaussian noise $\omega_i\in \mathbb R^{n_x}$ in the model of agents, i.e., $\omega_i \sim \mathcal N(0,W_i)$ with a diagonal covariance matrix $W_i\in \mathbb R^{n_x\times n_x}$. 
Due to the existence of probabilistic distribution of the initial state vector $x_{i|0}$, the initial state vector $x_{i|0} \in \mathcal N(\hat x_{i|0}, P_i)$, where $\hat x_{i|0}$ denotes the mean of the initial state vector, and $P_i$ is the covariance of the Gaussian distribution of the initial state vector.


In the following, we use the discrete-time dynamics model to illustrate our approach. In detail, the double-integrator dynamics equation is used to represent the dynamics of each agent, which can be described by
\begin{IEEEeqnarray}{rCl}
x_{i|k+1} &=& Ax_{i|k} + B u_{i|k} +\omega_i
\end{IEEEeqnarray}
where
\begin{IEEEeqnarray*}{c}
A = \begin{bmatrix}
	1 & 0 & \tau_s & 0 \\ 0 & 1 & 0 & \tau_s \\ 0 & 0 & 1 & 0 \\ 0 & 0 & 0 & 1
\end{bmatrix},
B = \begin{bmatrix}
	\frac{1}{2}\tau_s^2 & 0 \\ 0 & \frac{1}{2}\tau_s^2 \\ \tau_s & 0 \\ 0 & \tau_s
\end{bmatrix},
\end{IEEEeqnarray*}
and $\tau_s$ denotes the sampling time interval of the agents.

\subsection{Chance Constraints}
In this paper, the agents aim to avoid the static obstacles and other agents simultaneously. 

\subsubsection{Collision Avoidance with Static Obstacles}
We consider polygonal static obstacle regions, which all agents need to avoid effectively. A static obstacle region $O_o, \forall o \in \mathbb N_o$ is shown as the pink polygon in Fig.~\ref{fig:static_obs}, and $\mathbb N_o$ denotes the set of all obstacles the agents need to avoid. Thus, for one static obstacle $O_o$, the feasible region of position variable $p_i$ should be the complement of the region $O_o$, which is represented as the green shadow region in Fig.~\ref{fig:static_obs}. In detail, the feasible region of position variable is defined as a disjunction of linear constraints as follows:
\begin{IEEEeqnarray}{rCl}
O_{o}^\complement &\iff& \bigcup_{n=1}^{n_o} A_{o,n}^\top p_i \geq b_{o,n},	
\end{IEEEeqnarray}
where $A_{o,n}$ is the normal vector of the $n$th linear constraint, $b_{o,n}$ denotes the scalar corresponding to the $n$th linear constraint, the subscript $\cdot_n$ denotes the index of the linear constraint with $n\in \{1,2,\cdots,n_{o}\}$, and $n_{o}$ is the number of linear constraints defining the static obstacle region $O_o$. For example, the static obstacle $O_o$ in Fig.~\ref{fig:static_obs} is a quadrilateral, and thus $n_o=4$ in this case.

\begin{figure}[th]
	\centering
	\includegraphics[width=0.27\textwidth, trim=0 5 0 10,clip]{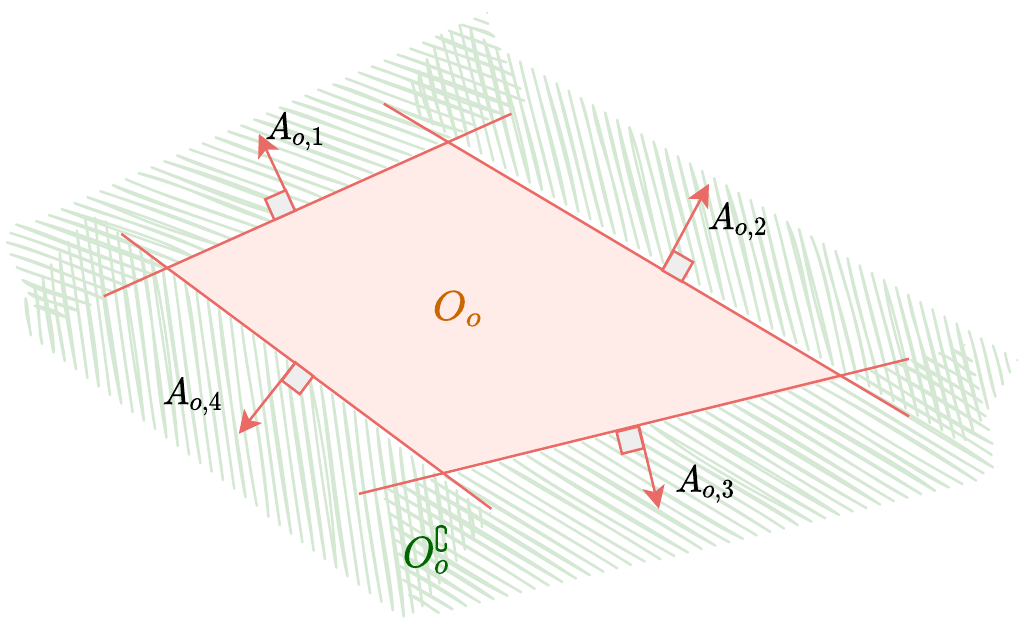}
	\caption{Feasible region for the position vector (in green color).}
	\label{fig:static_obs}
\end{figure}
For multi-agent systems, the constraints related to the position vector $p_i$ for all agents can be formulated as
\begin{IEEEeqnarray}{rCl}
\label{eq:static_obs_p_i} 
	p_i &\in& \bigcap_{o\in \mathbb N_o}  O_{o}^\complement, \quad \forall i \in \mathbb N_n, 
\end{IEEEeqnarray}
where $\bigcap\limits_{o\in \mathbb A_o} O_o^\complement$ denotes the composite feasible region for position state $p_i$. Here,~\eqref{eq:static_obs_p_i} denotes the position vector $p_i$ of the $i$th agent should belong to the conjunction of all feasible regions $O_{o}, \forall o \in \mathbb N_o$ for all agents.  

\subsubsection{Velocity Obstacles Concerning Other Agents}
Here, we use a velocity obstacle based method to achieve the collision-avoidance task. Thus, we can obtain the conic convex velocity obstacles that the agent should avoid. 
As we mentioned in Section~\ref{section:preliminary},  the velocity obstacle $VO_{ij}$ for the agent $i$ induced by the agent $j$ is represented as the pink conic region as shown in Fig.~\ref{fig:obstacle_demo}. In this figure, the feasible region of the velocity vector $v_i$ for the $i$th agent concerning the $j$th agent, denoted as $VO_{ij}^\complement$, is also represented as the green shadow region. In this figure, it can be observed that $VO_{ij}$ and $VO_{ij}^\complement$ are the conjunction and disjunction of two linear constraints, respectively. In detail, $VO_{ij}$ and $VO_{ij}^\complement$ are defined as follows:
\begin{IEEEeqnarray*}{rCl}
\label{eq:velocity_obstacle_describe}
VO_{ij} &\iff& \bigcap_{m\in \{1,2\}} N_{ij,m}^{\top} v_i \leq c_{ij,m} \IEEEyesnumber \IEEEyessubnumber \label{subeq:VO}\\
VO_{ij}^\complement &\iff& \bigcup_{m\in \{1,2\}} N_{ij,m}^{\top} v_i \geq c_{ij,m}, \IEEEnonumber \IEEEyessubnumber \label{subeq:VO_complement}
\end{IEEEeqnarray*}
where $N_{ij,m}$ denotes the normal vector of the two linear constraints which form the conic velocity obstacle $VO_{ij}$ for the agent $i$ induced by the agent $j$, $c_{ij,m}$ is the scalar corresponding to the $m$th linear constraint of $VO_{ij}$, and the subscript $\cdot_{m}$ denotes the index of the corresponding linear constraint of $VO_{ij}$ with $m\in \{1,2\}$. 
\begin{figure}[th]
	\centering
	\includegraphics[width=0.26\textwidth, trim=40 55 60 0,clip]{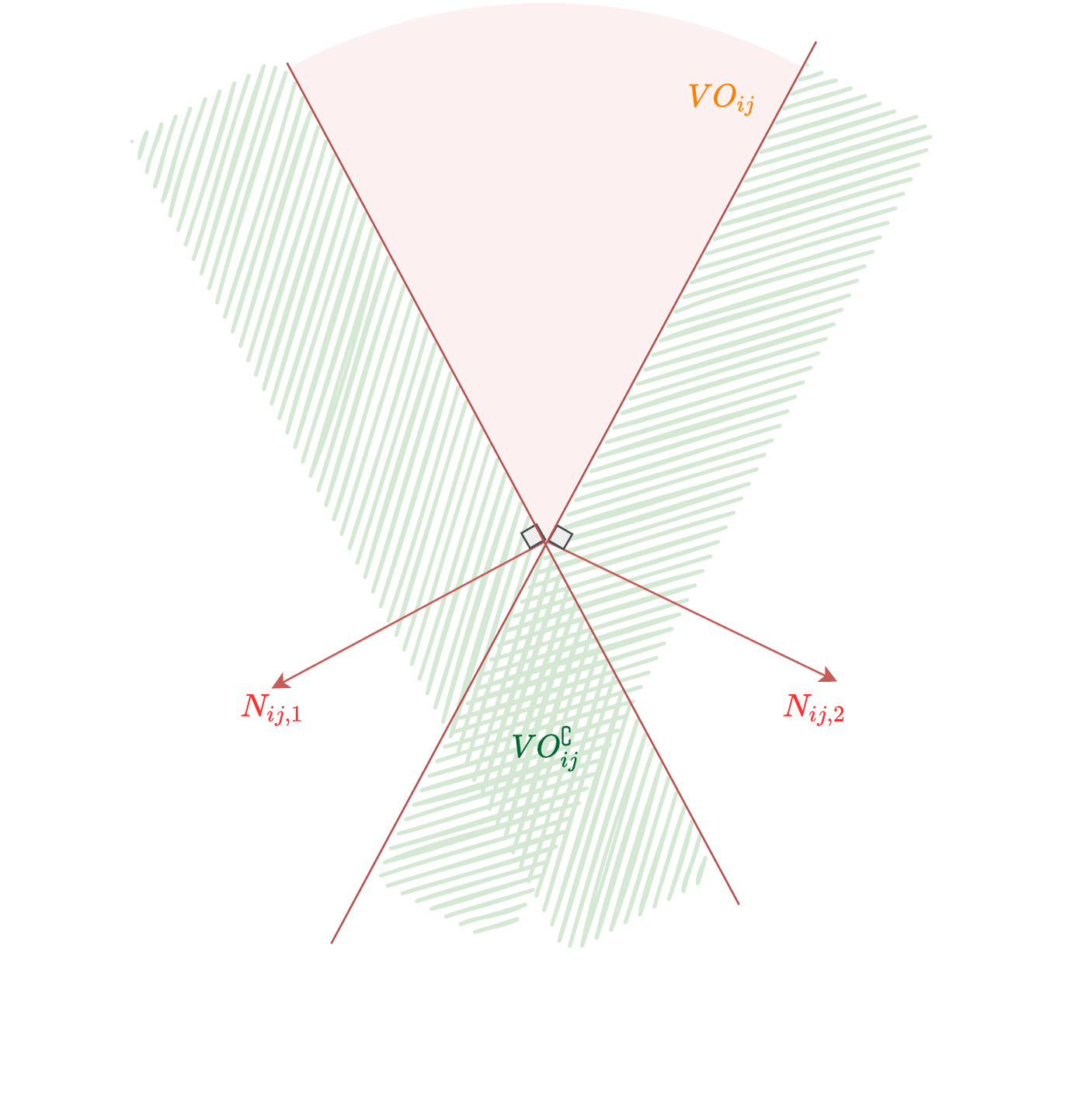}
	\caption{Feasible region for the velocity vector (in green color).}
	\label{fig:obstacle_demo}
\end{figure}

In multi-agent systems, we can obtain the composite velocity obstacles for the $i$th agent. Therefore, the collision-avoidance constraints can be formulated as
\begin{IEEEeqnarray}{rCl}
\label{eq:velocity_obstacle_for_v_i}
	           v_i \notin \bigcup\limits_{j\in \mathbb B_i} VO_{ij} 
	&\iff& v_i \in \bigcap\limits_{j\in \mathbb B_i} VO_{ij}^\complement, 
\end{IEEEeqnarray}
where $\bigcup\limits_{j\in \mathbb B_{i}} VO_{ij}$ denotes the composite velocity obstacles of the $i$th agent, $\bigcap\limits_{j\in \mathbb B_{i}} VO_{ij}^\complement$ represents the feasible region of velocity state $v_i$, and $\mathbb B_{i}$ means the set of neighbors of the $i$th agent. Here,~\eqref{eq:velocity_obstacle_for_v_i} represents the velocity vector $v_i$ of the $i$th agent should not belong to the composite velocity obstacles, which means $v_i$ should be inside the complement of the composite velocity obstacles.

\subsubsection{Chance Constraints of Static Obstacles and Velocity Obstacles }
Due to the existence of Gaussian noise in the state vector of agents, the predicted state vector is subject to the Gaussian distributions, i.e., $x_{i|k} = \hat x_{i|k} + \omega_i \in \mathcal N(\hat x_{i|k}, \Sigma_{i|k})$, where $\hat x_{i|k}$ and $\Sigma_{i|k}$ denote the mean and covariance of state vector $x_{i|k}$ at the timestamp $k$ in the prediction horizon. As for the position and velocity vector $p_{i|k}$ and $v_{i|k}$, it is straightforward that $p_{i|k} = L_p x_{i|k}$ and $v_{i|k} = L_v x_{i|k}$, where $L_p$ and $L_v$ are used to extract the position vector $p_i$ and velocity vector $v_i$ from the state vector $x_i$, respectively. 

Hence, the collision-avoidance constraints can be described in a probabilistic manner. For the $i$th agent, the chance constraints can be expressed as
\begin{IEEEeqnarray}{l}
\label{eq:chance_constraint}
	\operatorname{Pr}\left(p_{i|k} \in \bigcap\limits_{o\in \mathbb N_o} O_{o}^\complement \right) \ge  1- \epsilon_i \IEEEyesnumber\IEEEyessubnumber
	\label{subeq:chance_constraint_pos} \\
	\operatorname{Pr}\left(v_{i|k} \in \bigcap\limits_{j\in\mathbb B_i} VO_{ij|k}^{\complement}\right) \ge  1- \delta_i,  \IEEEnonumber\IEEEyessubnumber
	\label{subeq:chance_constraint_velo} 
\end{IEEEeqnarray}
where $\epsilon_i$ and $\delta_i$ are the collision risk bound of the position vector $p_i$ and velocity vector $v_i$ for the $i$th agent, respectively. \eqref{subeq:chance_constraint_pos} means the collision risk of $p_i$ which belongs to the composite of the feasible regions is no less than $1-\epsilon_i$, and \eqref{subeq:chance_constraint_velo} represents the collision risk between the $i$th agent and its neighbors should not be less than $1- \delta_i$. 

\subsection{Distributed Collision-Avoidance Problem}
Then, we can formulate the motion planning task for a stochastic multi-agent system as a distributed discrete-time chance-constrained optimization problem for each agent $i \in \mathbb N_n$ on $N$ prediction timestamps with a sampling time $\tau_s$. The objective is to determine the optimal trajectories and control inputs for all agents, such that these agents can move from their initial states to the target states while maintaining the collision risk below the given risk bound. 
For each ego agent $i$, given the position of the other agents $p_j, \forall j \in \mathbb B_i$, the initial probabilistic state $x_i^0 \in \mathcal N(\hat x_{i|0}, P_{i|0})$, the reference state vector $x_{\text{ref},i}$ and the collision risk bounds $\delta_i$ and $\epsilon_i$, the distributed  chance-constrained optimization problem for each agent is defined as
\begin{IEEEeqnarray*}{cCl}
	\label{eq:ori_problem}
	\min\limits_{\substack{x_{i|1:N}\\ u_{i|0:N-1}}} & \quad & \sum_{k=0}^{N-1} 
	J_{i|k}(x_{i|k},u_{i|k}) + J_{i|N}(x_{i|N}) 
			\IEEEyesnumber \IEEEyessubnumber \label{subeq:obj1}\\
	\operatorname{s.t.} & \quad & x_{i|k} = g_i\left(x_{i|k-1}, u_{i|k-1}\right) + \omega_{i|k-1} 
			\IEEEnonumber \IEEEyessubnumber \label{subeq:dyn1}\\
	&\quad& \omega_{i|k-1} \sim \mathcal N\left(0, W_i\right) 
			\IEEEnonumber \IEEEyessubnumber \label{subeq:noise1}\\
	&\quad& x_{i|0}\sim\mathcal N(\hat x_{i|0}, P_0) 
			\IEEEnonumber \IEEEyessubnumber \label{subeq:init1}\\
	&\quad& \operatorname{Pr}\left(p_{i|k} \in \bigcap\limits_{o\in \mathbb N_o} O_{o}^\complement \right) \ge  1-\epsilon_i 
			\IEEEnonumber \IEEEyessubnumber \label{subeq:static_obs1}\\
	&\quad& \operatorname{Pr}\left(v_{i|k} \in \bigcap\limits_{j\in \mathbb B_i} VO_{ij|k}^\complement\right) \ge  1-\delta_i 
			\IEEEnonumber \IEEEyessubnumber \label{subeq:vo1}\\
	&\quad& x_{\min,i} \leq x_{i|k} \leq x_{\max,i} 
			\IEEEnonumber \IEEEyessubnumber \label{subeq:xlim1}\\
	&\quad& u_{\min,i}   \leq u_{i|k-1}   \leq u_{\max,i} 
			\IEEEnonumber \IEEEyessubnumber \label{subeq:ulim1}\\
	&\quad& \forall k\in \mathbb Z_1^{N}, 
\end{IEEEeqnarray*}
where \eqref{subeq:obj1} contains two cost terms of states and control inputs, i.e., the cost-to-go term $J_{i|k}(x_{i|k},u_{i|k}) = \left\|x_{i|k}-x_{\mathrm{ref},i|k} \right\|_{Q_{i}} + \left\|u_{i|k}\right\|_{R_{i}}$, and the terminal cost term $J_{i|N}(x_{i|N}) = \left\|x_{i|k}-x_{\mathrm{ref},i|k} \right\|_{Q_{i|N}}$, $Q_i$ and $R_i$ are weighting matrices to penalize the deviation from the reference states and the unnecessary large control inputs, respectively. \eqref{subeq:xlim1} and~\eqref{subeq:ulim1} represent the limitation constraints of $x_i$ and $u_i$, and $x_{\min,i}, x_{\max,i}, u_{\min,i}, u_{\max,i}$ are the lower bound and upper bound of state vector and control input vector for the $i$th agent.


\begin{remark}
In Problem~\eqref{eq:ori_problem}, the nonconvex chance constraints~\eqref{subeq:static_obs1} and~\eqref{subeq:vo1} are difficult to be addressed. The evaluation of a chance constraint needs to compute the integral of a multivariate Gaussian distribution on a nonconvex region; thus, it is hard to obtain a closed-form solution to the chance-constrained problem. Also, some approximation strategies like sampling may suffer from computational efficiency and raise approximation errors. Besides, the value of the integral is nonconvex for the decision variables, due to the existence of disjunctions in $\bigcup_{j\in \mathbb B_i} VO_{ij}$ and $\bigcup_{o\in \mathbb N_o} O_o$. Consequently, the optimization problem is generally intractable. 
\end{remark}

\section{Problem Transformation}
\label{section:problem_transformation}
In Problem~\eqref{eq:ori_problem}, the collision-avoidance constraints regarding the static obstacles $O_o, \forall o\in \mathbb N_o$ have the similar property with the collision-avoidance constraints of velocity obstacles $VO_{ij}$ for other agents. Thus, in the following description, we use the transformation of the chance constraint~\eqref{subeq:vo1} as an example to demonstrate how we transform the probabilistic chance constraints as disjunctive deterministic constraints. 
As for~\eqref{subeq:static_obs1}, its transformation is much simpler than that of~\eqref{subeq:vo1}, because $VO_{ij|k}^\complement$ may be time-variant during the prediction horizon; however, $O_o^\complement$ is constant. 
The key difference between the transformation of the two constraints~\eqref{subeq:vo1} and~\eqref{subeq:static_obs1} lies in the difference in the number of linear constraints that formulate the velocity obstacles and static obstacles, respectively. 
In detail, the velocity obstacles are conic regions consisting of two linear bounds. The static obstacles are polyhedral regions formed by different numbers of linear bounds that depend on their specific shapes.
Therefore, the distributed optimization problem for the demonstration that omits the chance constraints of static obstacles~\eqref{subeq:static_obs1} is defined as
\begin{IEEEeqnarray*}{cCl}
	\label{eq:demo_problem}
	\min\limits_{\substack{x_{i|1:N}\\ u_{i|0:N-1}}} & \quad & \sum_{k=0}^{N-1} 
	J_{i|k}(x_{i|k},u_{i|k}) + J_{i|N}(x_{i|N}) \\
	\operatorname{s.t.} & \quad & \eqref{subeq:dyn1},~\eqref{subeq:noise1},~\eqref{subeq:init1},~\eqref{subeq:vo1},~\eqref{subeq:xlim1},~\eqref{subeq:ulim1} \\
	&\quad& \forall k\in \mathbb Z_1^{N}.  \yesnumber
\end{IEEEeqnarray*}
Since the difficulty in solving the optimization problem~\eqref{eq:demo_problem} lies in dealing with the probabilistic chance constraint~\eqref{subeq:vo1}, we transform nonconvex probabilistic chance constraint into a set of individual deterministic chance constraints.

\subsection{Analysis on Velocity Obstacle}
The chance constraint of velocity obstacles~\eqref{subeq:vo1} means that the risk of collision with all neighbors for the $i$th agent should be less than $\delta_i$; however, it is hard to handle it. 
Thus, we start to analyze one velocity obstacle $VO_{ij|k}$. All of the pertinent variables and their relationships are demonstrated in Fig.~\ref{fig:vo_analysis}. 
The green shadow area in this figure represents the feasible region of $v_i$ for the $i$th agent to avoid the collision with neighbors $B_i$.

\begin{figure}[th]
	\centering
	\includegraphics[width=0.45\textwidth, trim=30 70 0 0,clip]{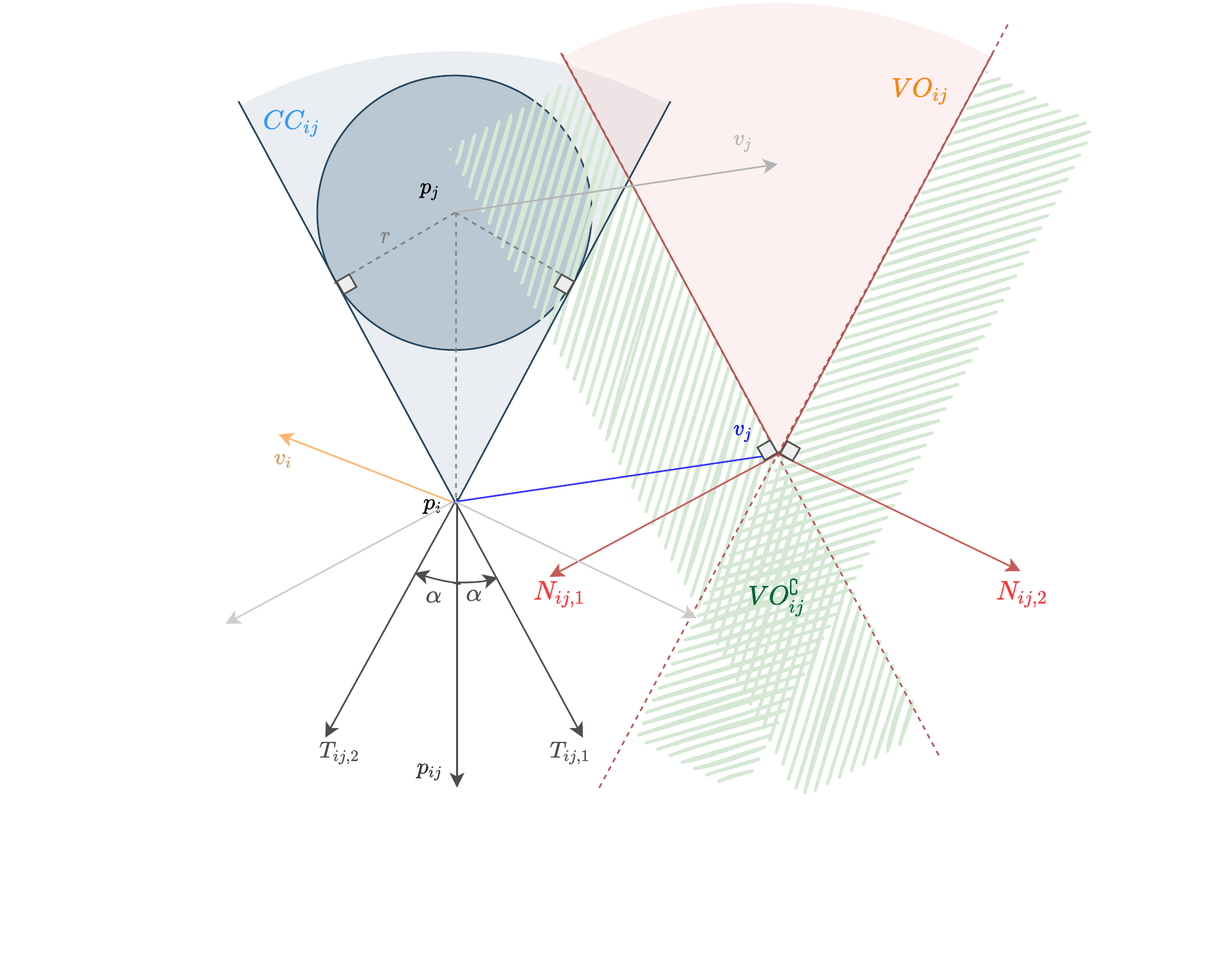}
	\caption{Feasible region analysis based on the velocity obstacle $VO_{ij}$}
	\label{fig:vo_analysis}
\end{figure}

Define the vector $p_{ij|k}:=p_{i|k}-p_{j|k} \in \mathbb R^{n_p}$, where $n_p$ is the dimension of the position vector, and the radius of the circle in Fig.~\ref{fig:vo_analysis} is given by $r=r_i+r_j$. 
Then, the two angles between $p_{ij|k}$ and the two boundary lines of the collision cone $CC_{ij|k}$ are the angle $\alpha_{ij|k}$ and $-\alpha_{ij|k}$ respectively, where $\alpha_{ij|k} \in \left[-\frac{\pi}{2}, \frac{\pi}{2}\right]$, as shown in Fig.~\ref{fig:vo_analysis}. 
It is straightforward that $\sin\alpha_{ij|k} = \frac{r}{\|p_{ij|k}\|}$, and $\cos \alpha_{ij|k} = \sqrt{1- \sin^2 \alpha_{ij|k}} = \frac{\sqrt{\|p_{ij|k}\|^2-r^2}}{\|p_{ij|k}\|}$. 
Thus, we can obtain the two tangent vectors $T_{ij,1|k}$ and $T_{ij,2|k}$ of the $CC_{ij|k}$, where
\begin{IEEEeqnarray*}{C}
	T_{ij,1|k} =  R( \alpha_{ij|k}) p_{ij|k}, \quad T_{ij,2|k} =  R(-\alpha_{ij|k}) p_{ij|k}, \IEEEeqnarraynumspace \yesnumber
\end{IEEEeqnarray*}
where $R(\theta)$ is a rotation matrix to rotate a vector with the degree $\theta$ in Euclidean space counterclockwise. Here, the tangent vectors $T_{ij,1|k}$ and $T_{ij,2|k}$ are obtained by rotating the vector $p_{ij|k}$ counterclockwise and clockwise with the degree of $\alpha_{ij|k}$, respectively, as shown in Fig.~\ref{fig:vo_analysis}.

In addition, according to Fig.~\ref{fig:vo_analysis}, we can also compute the normal vector $N_{ij,1|k}$ by rotating the tangent vectors $T_{ij,1|k}$ clockwise with the degree $\frac{\pi}{2}$. 
Meanwhile, another normal vector $N_{ij,2|k}$ can be derived by rotating the tangent vectors $T_{ij,2|k}$ counterclockwise with the degree $\frac{\pi}{2}$. Thus, $N_{ij,1|k}$ and $N_{ij,2|k}$ are defined by
\begin{IEEEeqnarray*}{rCl}
	N_{ij,1|k} = R\left(-\frac{\pi}{2}\right)T_{ij,1|k}, \quad
	N_{ij,2|k} = R\left( \frac{\pi}{2}\right)T_{ij,2|k}. \IEEEeqnarraynumspace \yesnumber
\end{IEEEeqnarray*}
Besides, the velocity obstacle $VO_{ij|k}$ can be obtained by translating the collision cone $CC_{ij|k}$ via the Minkowski sum $VO_{ij|k} = CC_{ij|k} \oplus v_{j|k}$. Therefore, the directions of two normal vectors of $VO_{ij}$ are the same as the directions of two normal vectors of $CC_{ij}$.

The relationship between the velocity $v_{i|k}$ and the velocity obstacle $VO_{ij|k}$ with respect to the $j$th agent is shown in Fig.~\ref{fig:vo_analysis}. 
As aforementioned, 
the velocity $v_{i|k}$ satisfying the constraint $v_{i|k} \notin \bigcup_{j\in \mathbb B_i} VO_{ij|k}$ can result in a collision-free trajectory for the $i$th agent. Hence, the feasible region of $v_{i|k}$ is the complement of the velocity obstacle $VO_{ij|k}$ regarding the $j$ agent and can be represented as the green shadow region in Fig.~\ref{fig:vo_analysis}. Hence, we can use the two linear constraints to represent the feasible region $VO_{ij|k}^\complement$ that is represented as~\eqref{eq:velocity_obstacle_describe}, in which the two scalars $c_{ij,1|k}$ and $c_{ij,2|k}$ of the two linear constraints can be given by
\begin{IEEEeqnarray*}{rCl}
c_{ij,1|k} = N_{ij,1|k} \cdot v_{j|k}, \quad
c_{ij,2|k} = N_{ij,2|k} \cdot v_{j|k}. \IEEEeqnarraynumspace \yesnumber
\end{IEEEeqnarray*}

Therefore, we can rewrite the collision-avoidance chance constraint in velocity space, i.e., ~\eqref{subeq:chance_constraint_velo}, as 
\begin{IEEEeqnarray}{rCl}
	\label{eq:prob_disjunct_constraints}
	\operatorname{Pr} \left( \bigcap\limits_{j\in \mathbb B_i} \bigcup\limits_{m\in\{1,2\}} \left( N_{ij,m|k}^{\top} v_{i|k} \geq c_{ij,m} \right) \right)	& \geq & 1-\delta_i.
\end{IEEEeqnarray}

\begin{remark}
In addition to the probabilistic property of~\eqref{eq:prob_disjunct_constraints}, another difficulty in handling~\eqref{eq:prob_disjunct_constraints} is that the disjunctions in~\eqref{eq:prob_disjunct_constraints} render the feasible region nonconvex. 
\end{remark}


\subsection{Propagation of Linear Gaussian Distribution}
As mentioned in Section~\ref{section:model}, we assume that the initial states subject to a multivariate Gaussian distribution $x_{i|0} \in \mathcal N(\hat x_{i|0}, P_{i|0})$. Besides, we also assume that the agent local dynamics are linear and there are Gaussian noises corresponding to model uncertainty and disturbances. Thus, the distribution of the future state is also subject to a Gaussian distribution, i.e., $\operatorname{Pr}(x_{i|k}|u_{i|0}, u_{i|1}, \cdots, u_{i|k-1})\sim \mathcal N(\hat x_{i|k}, \Sigma_{i|k})$. Based on the recursive linear system dynamics, the distribution of the future state can be calculated by
\begin{IEEEeqnarray*}{rCl}
	\hat x_{i|k} &=& \sum_{l=0}^{k-1} A_i^{k-l-1} B u_{i|l} + A_i^{k}\hat x_{i|0} \IEEEyesnumber \IEEEyessubnumber \label{subeq:mean}\\
	\Sigma_{i|k} &=&  \sum_{l=0}^{k-1} A_i^l W_i (A_i^{\top})^{l} + A_i^{k} P_{i|0} ({A_i^\top})^{k}. \IEEEnonumber \IEEEyessubnumber \label{subeq:covariance}
\end{IEEEeqnarray*}
\begin{remark}
	According to \eqref{subeq:mean}, the mean of the future state $\hat x_{i|k}$ at the timestamp $k$ in the prediction horizon is linear in the control inputs $u_{i|0}, u_{i|1}, \cdots, u_{i|{k-1}}$. Besides, the covariance of the future state $\Sigma_{i|k}$ is not a function of the control inputs $u_{i|0}, u_{i|1}, \cdots, u_{i|{k-1}}$, which means that the covariance during prediction is known for a given initial state covariance $P_{i|0}$ and given noise covariance $W_i$.
\end{remark}

\subsection{Risk Bound of Disjunctive Constraints}
In order to determine the probabilistic chance constraints in a deterministic manner, the following lemmas are introduced.

\begin{lemma}
\label{lemma:single_variable}
	A chance constraint on a single-variate Gaussian random variable $x\sim \mathcal N(\mu,\sigma^2)$ with fixed variance but varying mean can be translated into a deterministic constraint as
	\begin{IEEEeqnarray}{rCl}
	\operatorname{Pr}(x<0) \leq \delta &\iff & u\geq \eta, 	
	\end{IEEEeqnarray}
	where $\delta$ is the predefined risk bound, and $\eta$ is given by
	\begin{IEEEeqnarray*}{rCl}
	\eta &=& \sqrt{2}\sigma \operatorname{erf}^{-1}(1-2\delta)	,
	\end{IEEEeqnarray*}
	and the error function $\operatorname{erf}$ is defined as
	\begin{IEEEeqnarray*}{rCl}
	\operatorname{erf}(z) &=& \frac{2}{\sqrt{\pi}} \int_0^z e^{-t^2} dt.
	\end{IEEEeqnarray*}
\end{lemma}

\begin{lemma} 
	\label{lemma:multi_variable}
	Given any matrix $A$ and scalar $b$, for a multivariate random variable $ X(t)$ corresponding to the mean $\mu(t)$ and covariance $\Sigma(t)$, the chance constraint related to the predefined allowable collision risk bound $\delta$ is defined as
	\begin{IEEEeqnarray*}{rCl}
	\label{eq:multi_variable}
	{\operatorname{Pr}\left({A}^\top {X}(t) < b \right)}  \leq \delta,   \yesnumber
	\end{IEEEeqnarray*}
	which is equivalent to a deterministic linear constraint
	\begin{IEEEeqnarray*}{rCl}
	\label{eq:chance_constraint_2}
	{{A}^\top \mu(t) - b} \ge \eta(\delta),  \yesnumber
	\end{IEEEeqnarray*}
	where 
	\begin{IEEEeqnarray*}{c}
		\eta(\delta) = {\sqrt{2{A}^\top \Sigma(t) {A}}\operatorname{erf}^{-1}(1-2\delta)}.
	\end{IEEEeqnarray*}
\end{lemma}

\begin{theorem}
\label{theorem:theorem_risk_bound}
Given $x_{i|k}\sim \mathcal N(\hat x_{i|k}, \Sigma_{i|k})$ and $v_{i|k} = L_v x_{i|k}$, the probabilistic chance constraint~\eqref{eq:prob_disjunct_constraints} holds only if 
\begin{IEEEeqnarray*}{c}
\label{eq:condition}
\bigcap_{j\in \mathbb B_i} \bigcup_{m\in \{1,2\}} \left( N_{ij,m|k}^\top L_v\hat x_{i|k} - c_{ij,m|k} \geq g_{ij,m|k}(\delta_{ij,m|k}) \right) \IEEEeqnarraynumspace \IEEEyesnumber \IEEEyessubnumber \label{eq:condition_1} \\
\sum_{j\in\mathbb B_i}\sum_{m\in\{1,2\}} \delta_{ij,m|k} \leq \delta_i, \IEEEeqnarraynumspace \IEEEnonumber \IEEEyessubnumber \label{eq:condition_2}
\end{IEEEeqnarray*}
where 
\begin{IEEEeqnarray*}{rCl}
	 g_{ij,m|k}(\delta) &=& \sqrt{2N_{ij,m|k}^{\top}L_v\Sigma_{i|k}L_v^\top N_{ij,m|k}} \operatorname{erf}^{-1}(1-2\delta).
\end{IEEEeqnarray*}
\end{theorem}

\begin{proof}
	It is evident that for any two events $A$ and $B$, the probability of occurrence of event $A$ and event $B$ simultaneously is equal to
\begin{IEEEeqnarray}{rCl}
	\operatorname{Pr} (A\bigcap B) &=& 1- \operatorname{Pr} (A^\complement \bigcup B^\complement).
\end{IEEEeqnarray}
where $\operatorname{Pr} (A^\complement \bigcup B^\complement)$ denotes the probability of occurrence of the complement of event $A$ or event $B$.
Thus, we have
\begin{IEEEeqnarray*}{rCl}
	&&\operatorname{Pr} \left( \bigcap\limits_{j\in \mathbb B_i} \bigcup\limits_{m\in\{1,2\}} \left( N_{ij,m}^{k\top} v_i^k \geq c_{ij,m}^k \right) \right) \\
	&=& 1- \operatorname{Pr} \left( \bigcup\limits_{j\in \mathbb B_i} \left( \bigcap_{m\in\{1,2\}} N_{ij,m}^{k\top} v_i^k \leq c_{ij,m} \right) \right). 
	\IEEEyesnumber
\end{IEEEeqnarray*}

Besides, for any two events $A$ and $B$, we have
\begin{IEEEeqnarray}{rCl}
\operatorname{Pr}(A \bigcup B) &\leq & \operatorname{Pr}(A) + \operatorname{Pr}(B),
\end{IEEEeqnarray}
where $\operatorname{Pr}(A)$ and $\operatorname{Pr}(B)$ are the probability of occurrence of event $A$ and event $B$.
Hence, for any number of events $A_i$, we can derive that
\begin{IEEEeqnarray}{rCl}
\label{eq:theorem_cup}
	\operatorname{Pr} \left( \bigcup_{i} A_i \right) &\leq & \sum_i \operatorname{Pr} (A_i).
\end{IEEEeqnarray}
On the other hand, for any two events $A$ and $B$, the probability of occurrence of events $A$ and $B$ simultaneously fulfills
\begin{IEEEeqnarray}{C}
\operatorname{Pr}(A \bigcap B) \leq  \operatorname{Pr}(A), \quad
\operatorname{Pr}(A \bigcap B) \leq \operatorname{Pr}(B).
\end{IEEEeqnarray}
Hence, for any number of events $A_i$, we can derive that
\begin{IEEEeqnarray}{rCl}
\label{eq:theorem_cap}
	\operatorname{Pr} \left( \bigcap_{i} A_i \right) &\leq & \operatorname{Pr} (A_i), \forall i.
\end{IEEEeqnarray}
According to \eqref{eq:theorem_cup} and~\eqref{eq:theorem_cap}, we can obtain that
\begin{IEEEeqnarray*}{rCl}
	&& \operatorname{Pr} \left( \bigcap\limits_{j\in \mathbb B_i} \bigcup\limits_{m\in\{1,2\}} \left( N_{ij,m|k}^{\top} v_{i|k} \geq c_{ij,m|k} \right) \right) \\
	&\geq & 1-\sum\limits_{j\in \mathbb B_i} \operatorname{Pr} \left( \bigcup\limits_{m\in\{1,2\}} \left( N_{ij,m|k}^{\top} v_{i|k} \geq c_{ij,m|k} \right) \right) \\
	&\geq & 1 - \sum_{j\in \mathbb B_i} \left( \operatorname{Pr}(N_{ij,m|k}^{k\top} v_{i|k} \geq c_{ij,m|k} ), \; \forall m\in \{1,2\} \right).
\end{IEEEeqnarray*}

Notice that $v_{i|k} = L_v x_{i|k}$, where $L_v$ is used to extract the velocity vector $v_{i|k}$ from the state vector $x_{i|k}$ for the $i$th agent at the timestamp $k$.
Based on Lemma~\ref{lemma:multi_variable}, we can derive that
\begin{IEEEeqnarray*}{l}
	N_{ij,m|k}^{\top} L_v \hat x_{i|k} - c_{ij,m|k} \geq g_{ij,m}(\delta_{ij,m|k}) \\
	\quad \Rightarrow  \operatorname{Pr}(N_{ij,m|k}^{\top} v_{i|k} < c_{ij,m|k}) \leq \delta_{ij,m|k} \\
	c_{ij,m|k} - N_{ij,m|k}^{\top} L_v \hat x_{i|k} \geq g_{ij,m}(\delta_{ij,m|k}) \\ 
	\quad \Rightarrow \operatorname{Pr}(N_{ij,m|k}^{\top} v_{i|k} > c_{ij,m|k}) \leq \delta_{ij,m|k}. \yesnumber
\end{IEEEeqnarray*}
Therefore, we can obtain the two conditions, i.e., \eqref{eq:condition_1} and \eqref{eq:condition_2}.
Hence, there exists $\hat m$ such that
\begin{IEEEeqnarray}{rCl}
	\sum\limits_{j\in \mathbb B_i} \operatorname{Pr}\left( N_{ij,\hat m|k}^{\top} \hat v_{i|k} \leq c_{ij,\hat m|k}  \right) \leq \sum_{j\in \mathbb B_i} \delta_{ij,m|k} \leq \delta_i.
\end{IEEEeqnarray}
Thus, the proof is completed.
\end{proof}
In Theorem~\ref{theorem:theorem_risk_bound}, the parameter $\delta_{ij,m|k}$ represents the risk allocating to each of the chance constraints with respect to the $j$th agent at the timestamp $k$. Besides, there exists $m \in \{1,2\}$ such that at least one of the linear constraint $N_{ij,m|k}^{\top} L_v \hat x_{i|k} - c_{ij,m|k} \geq g_{ij,m}(\delta_{ij,m|k})$ is activated. By ensuring the inequality~\eqref{eq:condition_2}, the overall collision risk for the $i$th agent can be guaranteed to be no greater than $\delta_i$. 
For simplicity, we define
\begin{IEEEeqnarray}{rCl}
	h_{ij,m|k}(x_{i|k}) &=& N_{ij,m|k}^{\top} L_v x_{i|k} - c_{ij,m}.
\end{IEEEeqnarray}
According to Theorem~\ref{theorem:theorem_risk_bound}, we can transform the original probabilistic collision-avoidance chance constraint concerning other agents, i.e., \eqref{subeq:vo1}, as deterministic disjunctive constraints. Thus, the optimization problem~\eqref{eq:demo_problem} can be transformed as 
\begin{IEEEeqnarray*}{cCl}
	\label{eq:problem_transform}
	\min\limits_{\substack{x_{i|1:N}\\ u_{i|0:N-1} \\ \delta_{i\forall j\forall m} }} & \quad & \sum_{k=0}^{N-1} 
	J_{i|k}(x_{i|k},u_{i|k}) + J_{i|N}(x_{i|N}) \\
	\operatorname{s.t.} & \quad & \eqref{subeq:dyn1},~\eqref{subeq:noise1},~\eqref{subeq:init1},~\eqref{subeq:xlim1},~\eqref{subeq:ulim1} \\
	&& \bigcap_{j\in \mathbb B_i} \bigcup_{m\in \{1,2\}} h_{ij,m|k}(\hat x_{i|k}) \geq g_{ij,m|k}(\delta_{ij,m|k}) \IEEEeqnarraynumspace\IEEEyesnumber \IEEEyessubnumber \label{subeq:disjunctive1}\\
	&& \sum_{j\in\mathbb B_i}\sum_{m\in\{1,2\}} \delta_{ij,m|k} \leq \delta_i \IEEEnonumber \IEEEyessubnumber \label{subeq:disjunctive2}\\
	&& \delta_{ij,m|k}\geq 0, \forall j\in \mathbb B_i, \forall m\in \{1,2\} \IEEEnonumber \IEEEyessubnumber \label{subeq:disjunctive3}\\
	&& \forall k\in \mathbb Z_1^N.
\end{IEEEeqnarray*}
Note that $\delta_{ij,m|k}, \forall j\in \mathbb B_i, \forall m\in\{1,2\}$ is introduced as an optimization variable in Problem~\eqref{eq:problem_transform}, which should satisfy~\eqref{subeq:disjunctive1},~\eqref{subeq:disjunctive2}, and~\eqref{subeq:disjunctive3}. 

\begin{corollary}
	\label{corollary1}
	Any feasible solution to Problem~\eqref{eq:problem_transform} is a feasible solution to Problem~\eqref{eq:demo_problem}.
\end{corollary}
\begin{proof}
	Theorem~\eqref{theorem:theorem_risk_bound} shows that the constraints,~\eqref{subeq:disjunctive1},~\eqref{subeq:disjunctive2} and~\eqref{subeq:disjunctive3}, indicate the complete probabilistic chance constraint~\eqref{subeq:vo1}. All other constraints are identical between Problem~\eqref{eq:problem_transform} and Problem~\eqref{eq:demo_problem}, which completes the proof.
\end{proof}

According to Corollary~\ref{corollary1}, the optimal sequence of control inputs computed by solving Problem~\eqref{eq:problem_transform} is feasible for the original optimization problem, i.e., Problem~\eqref{eq:demo_problem}, and we can address the problem~\eqref{eq:demo_problem} by solving Problem~\eqref{eq:problem_transform}, which is a disjunctive program~\cite{sherali2012optimization}. However, due to the introduction of variable $\delta_{ij,m|k}$, the computation time to solve Problem~\eqref{eq:problem_transform} could be high. 
Hence, we introduce a technique by setting $\delta_{ij,m|k}$ fixed to decrease the computation time substantially for real applications, i.e., setting $\delta_{ij,m|k}$ be fixed and equal for all $j,m$, which means
\begin{IEEEeqnarray}{rCl}
\delta_{ij,m|k} &=& \frac{\delta_i}{2|\mathbb B_i|},
\end{IEEEeqnarray}
where $|\mathbb B_i|$ denotes the number of elements in the set $\mathbb B_i$ and the value 2 denotes the number of elements in the set $\{1,2\}$. In this case, the conditions~\eqref{eq:condition_1} and~\eqref{eq:condition_2} are satisfied. Thus, we can formulate an optimization problem as
\begin{IEEEeqnarray*}{cCl}
	\label{eq:problem_simp}
	\min\limits_{\substack{x_{i|1:N}\\ u_{i|0:N-1}}} & \quad & \sum_{k=0}^{N-1} 
	J_{i|k}(x_{i|k},u_{i|k}) + J_{i|N}(x_{i|N}) \IEEEyesnumber \IEEEyessubnumber\\
	\operatorname{s.t.} & \quad & \eqref{subeq:dyn1},~\eqref{subeq:noise1},~\eqref{subeq:init1},~\eqref{subeq:xlim1},~\eqref{subeq:ulim1} \\
	&\quad& \bigcap_{j\in \mathbb B_i} \bigcup_{m\in \{1,2\}} h_{ij,m|k}(\hat x_{i|k}) \geq g_{ij,m|k}\left(\frac{\delta_i}{2|\mathbb B_i|}\right) \IEEEeqnarraynumspace \IEEEnonumber \IEEEyessubnumber\\
	&\quad& \forall k\in \mathbb Z_1^{N}.
\end{IEEEeqnarray*}
Obviously, the optimization problem~\eqref{eq:problem_simp} simplifies Problem~\eqref{eq:problem_transform} by setting the optimization variable $\delta_{ij,m|k}$ in~\eqref{eq:problem_transform} as a fixed and equal value. The feasible set of Problem~\eqref{eq:problem_transform} contains the feasible set of Problem~\eqref{eq:problem_simp}; thus, Problem~\eqref{eq:problem_simp} provides an upper bound to the optimal cost of Problem~\eqref{eq:problem_transform}.
\begin{corollary}
	\label{corollary2}
	A solution to Problem~\eqref{eq:problem_simp} is a feasible solution to Problem~\eqref{eq:problem_transform}, and also a feasible solution to Problem~\eqref{eq:demo_problem}.
\end{corollary}
\begin{proof}
	The proof of Corollary~\ref{corollary2} is similar to that of Corollary~\ref{corollary1}, and thus it is omitted. 
\end{proof}


\section{Problem Solving}
\label{section:problem_solving}
So far, Problem~\eqref{eq:problem_transform} can be simplified as Problem~\eqref{eq:problem_simp} 
by using the $\delta_{ij,m|k}$ setting technique to improve the computational efficiency. The two optimization problems are disjunctive programming problems. A general disjunctive program can be transformed as an equivalent optimization problem by introducing binary variables; thus, a mixed-integer programming problem can be formulated. The mixed-integer programming has been well studied and can be addressed by many highly optimized software~\cite{floudas2005mixed}. In summary, Problem~\eqref{eq:problem_transform} is the resulting optimization problem with disjunctive convex constraints~\eqref{subeq:disjunctive1}, and the risks $\delta_{ij,m|k}$ should satisfy~\eqref{subeq:disjunctive2} and~\eqref{subeq:disjunctive3}. Further, the mixed integer programming problem is transformed as
\begin{IEEEeqnarray*}{ccl}
\label{eq:problem_minlp}
	\min\limits_{\substack{x_{i|1:N}\\ u_{i|0:N-1} \\ z_{i|\forall j \forall m}}} & \quad & \sum_{k=0}^{N-1} \left\|x_{i|k+1}-x_{\mathrm{ref},i|k+1} \right\|_{Q_{i}} + \left\|u_{i|k}\right\|_{R_{i}} \IEEEyesnumber \IEEEyessubnumber\\
	\operatorname{s.t.} & \quad & \eqref{subeq:dyn1},~\eqref{subeq:noise1},~\eqref{subeq:init1},~\eqref{subeq:xlim1},~\eqref{subeq:ulim1},~\eqref{subeq:disjunctive2},~\eqref{subeq:disjunctive3} \\
	&\quad& g_{ij,m|k}(\delta_{ij,m|k}) - h_{ij,m|k}(\hat x_{i|k}) \leq M(1-z_{ij,m|k}) \\ \IEEEnonumber \IEEEyessubnumber  \label{subseq:binary1} \\
	&\quad& \sum\limits_{m=1}^2 z_{ij,m|k} \geq 1 \IEEEnonumber \IEEEyessubnumber \label{subeq:binary2}\\
	&\quad& z_{ij,m|k} \in \{0,1\} \IEEEnonumber \IEEEyessubnumber \label{subeq:binary3}\\
	&\quad& \forall k\in \mathbb Z_1^{N}, \forall j\in \mathbb B_i, \forall m\in \{1,2\},
\end{IEEEeqnarray*}
where $M$ is a large enough positive constant, $z_{ij,m|k}$ denotes the binary variables. The value of the binary variable $z_{ij,m|k}$ in~\eqref{subseq:binary1} determines whether the corresponding linear constraint is activated or not, and~\eqref{subeq:binary2} guarantees that  at least one constraint in each disjunction is imposed, as required. Therefore, the solution of the mixed-integer programming, i.e., \eqref{eq:problem_minlp}, can guarantee that the collision risk between the ego agent $i$ and all the other agents $\forall j\in \mathbb B_i$ is less than the risk bound $\delta_i$. Note that the variable $\delta_{ij,m|k}$ should be set as the fixed value $\frac{\delta_i}{2|\mathbb B_i|}$ when solving Problem~\eqref{eq:problem_simp}.

As for the collision avoidance with static obstacles, the transformation is similar. The probabilistic chance constraints of static obstacles can be transformed into disjunctive deterministic constraints, and then, be transformed into mixed-integer constraints by introducing the binary variables. Therefore, the related mixed integer constraints with respect to the avoidance of static obstacles can be generated similarly.

\subsection{Problem Setting}
During the robust motion planning process, we utilize the geometric information of velocity obstacles $VO_{ij|k}$, which requires the position and velocity information of the agents at the timestamp $k$ in the prediction horizon. 
However, the position and velocity information of agents $p_{i|k}$ and $v_{i|k}$ is hard to collect in the distributed method. 
A normal solution is to assume that the position and velocity $p_{\forall i|k}$ and $v_{\forall i|k}$ keep unchanged during the prediction horizon $k\in \mathbb Z_1^N$, which could incur the weak performance and produce a solution which is far from the optimal one. 
Thus, we need to use the time-variant information of $p_{\forall i|k}$ and $v_{\forall i|k}$ at the timestamp $k$ in the prediction horizon to construct the velocity obstacle $VO_{ij|k}$. In our setting, the predicted states $x_{\forall i|k}$ at current running step $t$ could be obtained through the assumption that the velocity of the agents remains fixed in the prediction horizon $k\in \mathbb Z_1^N$. 
An agent can infer the predicted states of its neighbors by just observing the current state of neighbors; thus, communication among agents is not necessary for this setting.

\begin{remark}
	Compared with the existing velocity obstacle method, our proposed method could be smoother owing to the MPC scheme, because the velocity obstacle method assumes the instantaneous velocity changes for the agents~\cite{van2011reciprocal}. 
	Besides, the velocity obstacle method requires the design of the preferred velocity, which an agent would take in the absence of other agents or obstacles and could be chosen manually or by some external algorithms~\cite{snape2011hybrid}. 
	However, this design regarding the preferred velocity is not required in our proposed method, which makes our approach favorable in real applications.
\end{remark}

\section{Results}
\label{section:results}
This section describes the implementation of the proposed method, and the effectiveness of the method is evaluated by simulations. All the relevant parameters of the simulation are shown in Table~\ref{tab:setting_params} in the Appendix. Here, we add the Gaussian noise to the states of agents model, and the added noise is zero mean with the covariance $W_i$. All of the simulations are implemented in Python 3.7 environment on a PC with Intel i5 CPU@3.30 GHz. 

\subsection{Multi-agent Systems Motion Planning w/o Static Obstacles}
In this simulation, there are 20 agents which need to reach their desired target positions. The means of initial positions of the 20 agents are uniformly located on a circle with a radius of 10 m. In this task, all agents' initial and target positions are symmetric along with the origin, i.e., both the $x$ and $y$ axis. For example, an agent whose mean of initial position is (10,~0)~m is required to reach its desired destination (-10,~0)~m. All agent are represented by a circle with radius $r_i = 0.1$ m and different colors. The prediction horizon $N$ is set as $N=25$, and the number of running steps is $t_{\mathrm{run}}=12.5$~s. The risk bound is $\delta_i=0.1$. Fig.~\ref{fig:snapshot} shows the snapshots of these agents in six different timestamps. According to this figure, we can observe that all agents can successfully arrive at their desired target positions without colliding with other agents. 
\begin{figure}[th]
	\centering
	\includegraphics[width=0.45\textwidth, trim=10 30 0 20,clip]{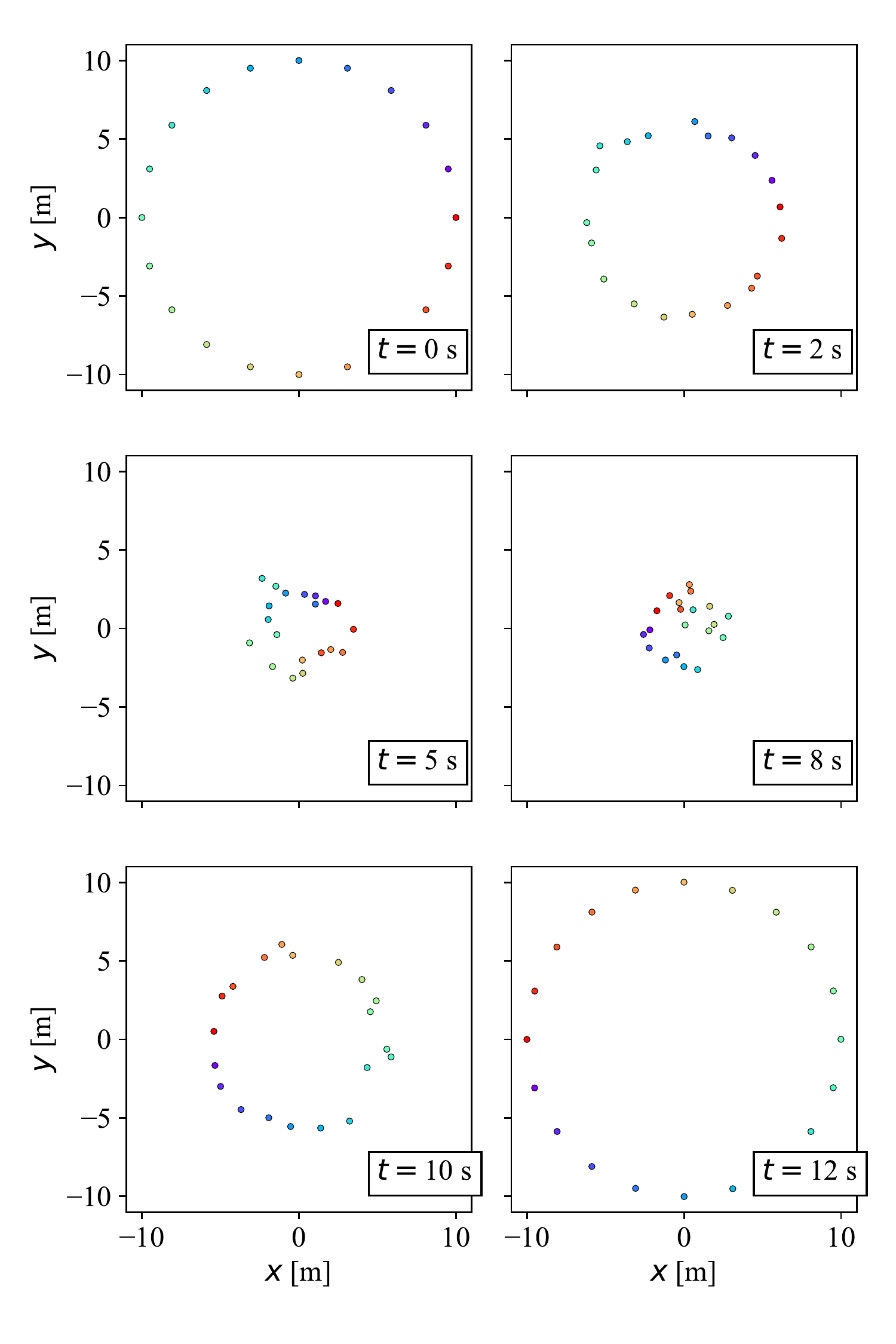}
	\caption{Snapshots of collision-free motions of 20 agents in six different timestamps.}
	\label{fig:snapshot}
\end{figure}

Fig.~\ref{fig:comp_time_with_diff_horizon} presents the relationship between the average computational time of each agent per running step and the varying prediction horizon $N$ of the distributed MPC problem. In this figure, the line represents the average computing time per agent per running step, and the filled region means the range of the related computing time. According to Fig.~\ref{fig:comp_time_with_diff_horizon}, with the prediction horizon $N$ increasing, the average computational time of each agent per execution increases. This figure also indicates that the average computational time of each agent per execution is still within the allowable range (0,~$\tau_s$]~s, even when we set $N=50$. 
\begin{figure}[th]
	\centering
	\includegraphics[width=0.4\textwidth, trim=0 20 0 10,clip]{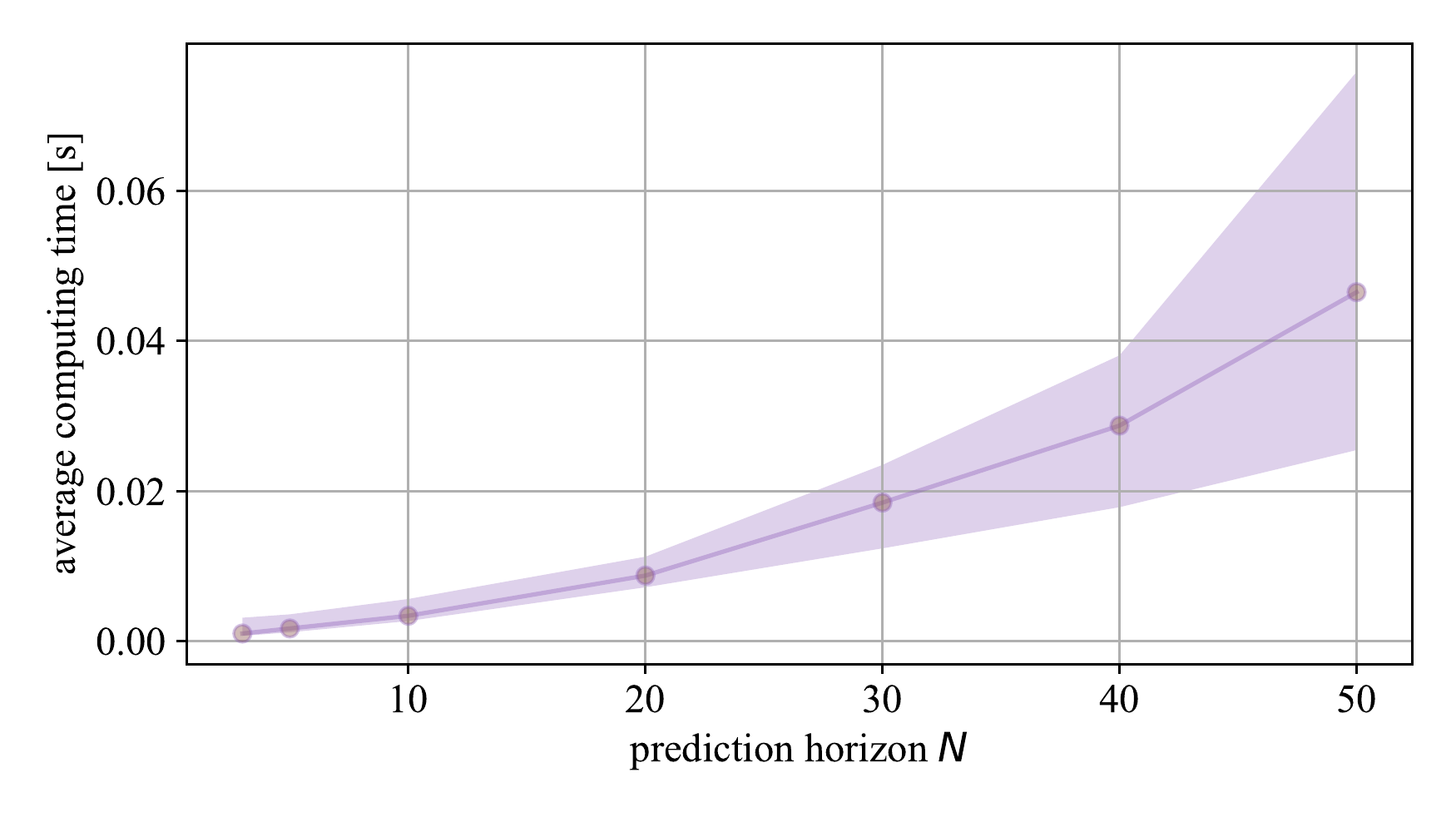}
	\caption{The relationship between the average computing time per agent per execution and the varying prediction horizon $N$.}
	\label{fig:comp_time_with_diff_horizon}
\end{figure}
Fig.~\ref{fig:comp_time_with_diff_number} shows the computational time with an increasing number of agents moving across the origin circle. In this case, similarly, all agents' initial positions and target positions are symmetric along with the origin. The prediction horizon is set as $N=10$. Based on Fig.~\ref{fig:comp_time_with_diff_number}, it is straightforward to observe that the computational time increases with the number of agents increasing. 
\begin{figure}[th]
	\centering
	\includegraphics[width=0.4\textwidth, trim=0 15 0 10,clip]{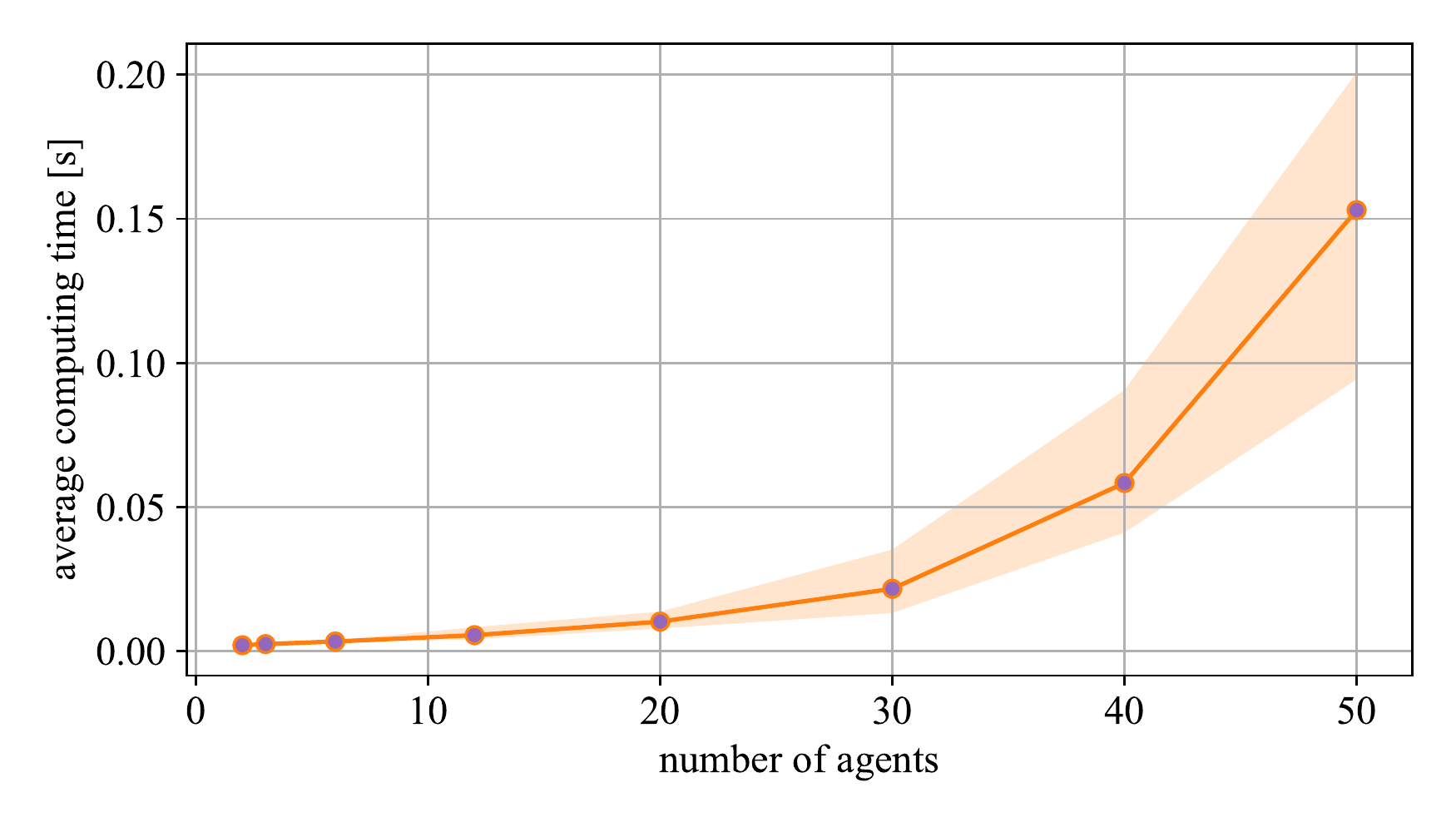}
	\caption{The average computing time per running step per agent of different numbers of agents (the prediction horizon is N=10).}
	\label{fig:comp_time_with_diff_number}
\end{figure}

In order to show the effectiveness of our proposed method, the map, as shown in Fig.~\ref{fig:traj_with_static_obs}, is designed to be used as an application environment. This map in this figure is challenging for chance-constrained motion planning methods, since there is a narrow corridor when the agent passes from the initial position (0,~0)~m to the target position (0,~10)~m. The initial and target positions are shown as the orange and blue circles with a radius of the agent radius $r_i=0.2$~m, respectively. Fig.~\ref{fig:traj_with_static_obs} shows an example of the trajectories planned by our method with $\epsilon_i = 0.01$, represented by the orange line. The prediction horizon is set as $N=20$. For this demonstrated trajectory in Fig.~\ref{fig:traj_with_static_obs}, it is obvious that our proposed method can plan collision-free trajectories for all agents successfully.
\begin{figure}[th]
	\centering
	\includegraphics[width=0.4\textwidth, trim=0 20 0 10,clip]{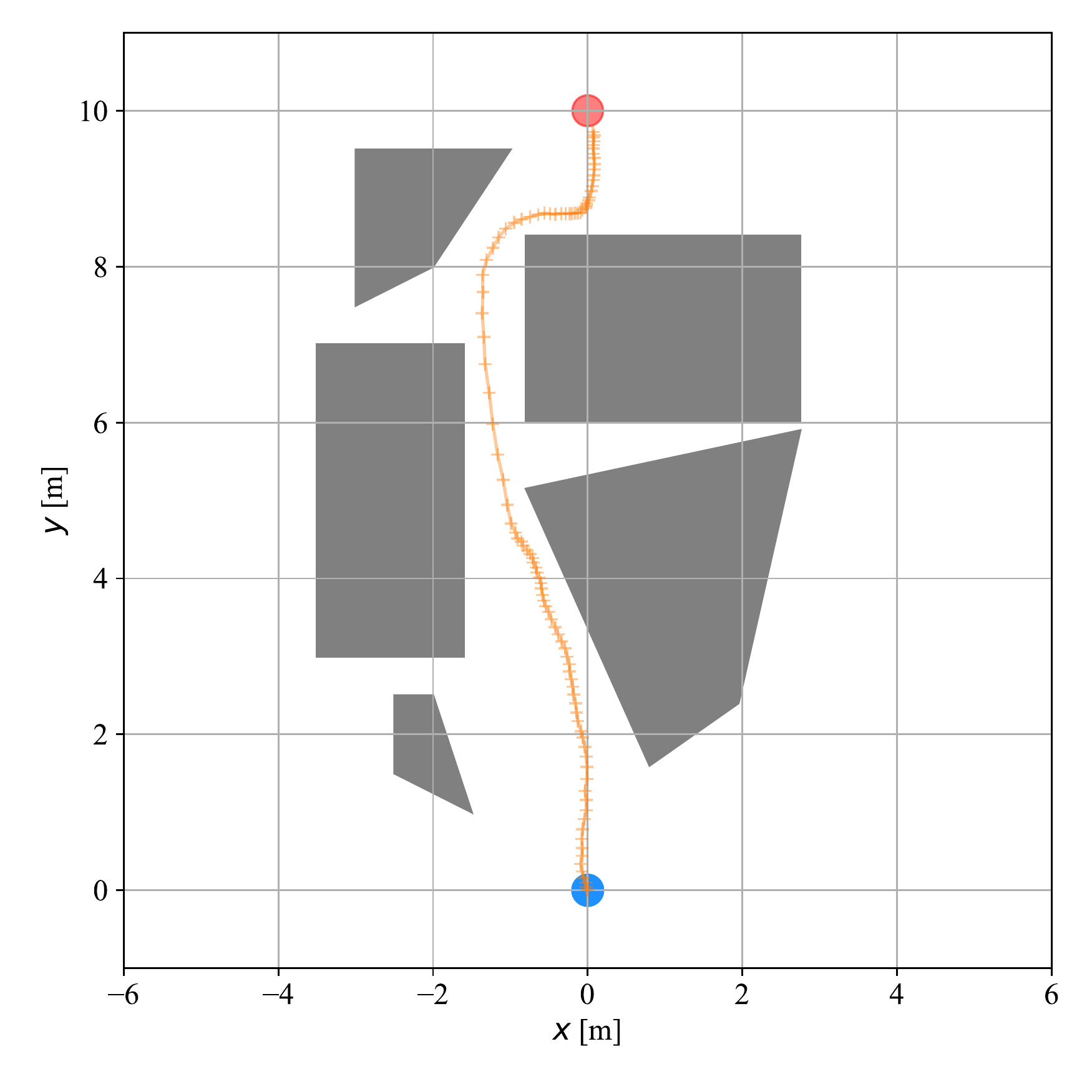}
	\caption{Example resulting trajectory in the designed map with $\epsilon=0.01$.}
	\label{fig:traj_with_static_obs}
\end{figure}
Besides, the value of the risk bounds $\delta_i$ and $\epsilon_i$ will significantly influence the resulting trajectory of the agent. 
Fig.~\ref{fig:traj_with_static_obs_diff_delta} presents the planned trajectories of the agent with different risk bound value $\epsilon_i$, which is chosen as $0.1,\ 0.01,\ 0.001$, respectively. According to this figure, we can observe that the planned trajectory becomes more conservative with the predefined collision risk bound $\epsilon_i$ decreasing. Note that a small risk bound $\epsilon_i$ means a smaller allowable collision risk. Thus, we can derive that the value of risk bound will influence the conservatism of the planned trajectory, and a proper collision risk bound can balance conservatism and performance.
\begin{figure}[th]
	\centering
	\includegraphics[width=0.4\textwidth, trim=0 20 0 10,clip]{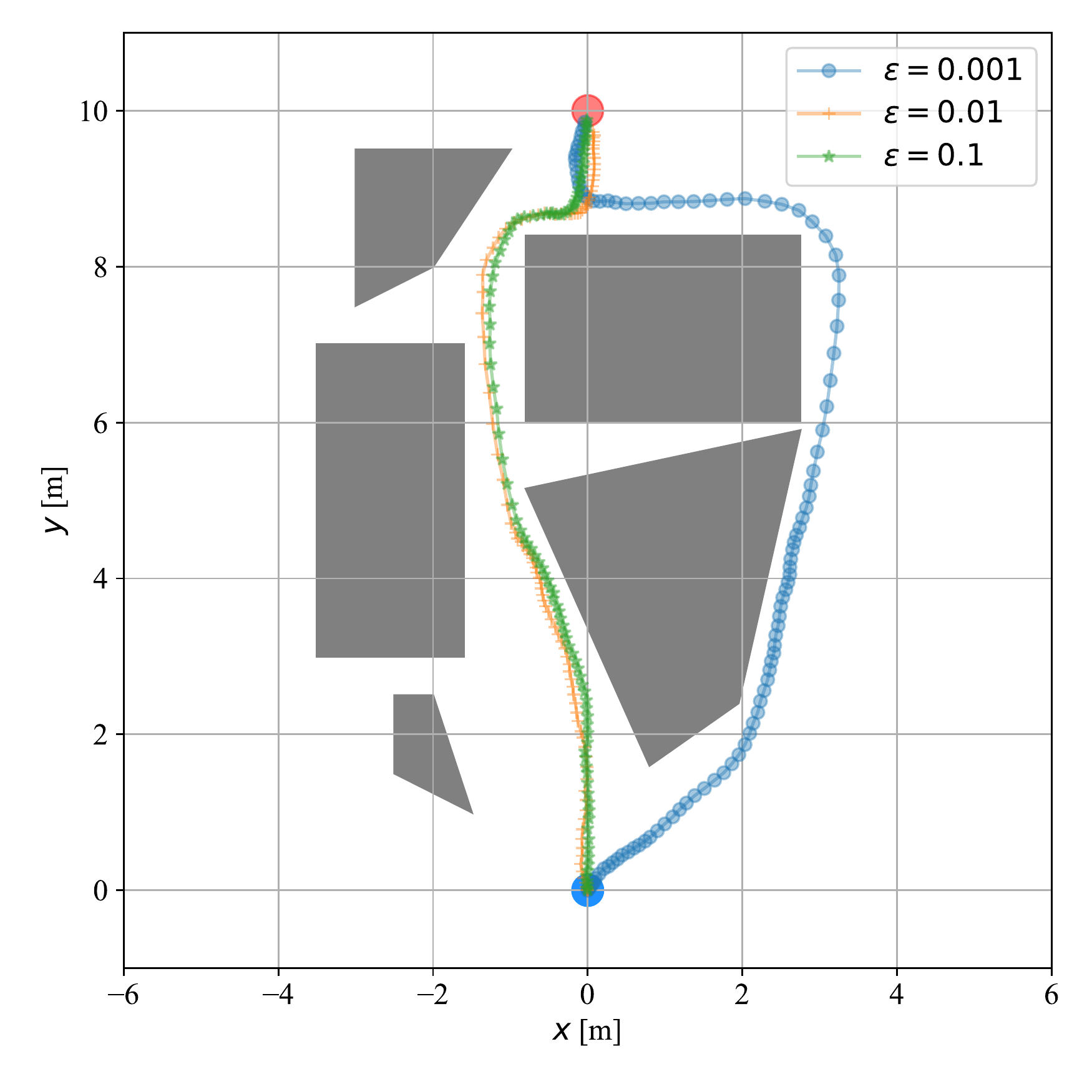}
	\caption{Resulting trajectories in the designed map for different $\epsilon_i$ values.}
	\label{fig:traj_with_static_obs_diff_delta}
\end{figure}

\subsection{Comparison with Velocity Obstacle Method}
\label{section:results_compare1}
In this section, we compare the performance of our proposed method and the pure velocity obstacle method, which is abbreviated as VO in the following text. The prediction horizon of our method is set as $N=20$, and the radius of the agent is set as $r_i=0.2$~m. The risk bound is $\delta_i=0.1$.
Here, 4 agents are utilized to achieve the robust motion planning task. The 4 agents are initially located at a circle with a radius of 4~m. In this task, all 4 agents are required to reach their desired destinations which are centrosymmetric to the initial positions. For example, the initial and target positions of the green agent are set as (4,~0)~m and (-4,~0)~m, respectively. The trajectories of the 4 agents by using the are illustrated in Fig.~\ref{fig:comp_traj_vo}. Both methods can navigate agents to reach their desired destinations without collision. However, our method could provide more smoother trajectories than the VO method. 
Fig.~\ref{fig:comp_velo_vo} demonstrates the change of velocities of one agent (randomly chosen from the 4 agents) with running step $t$ changing. According to this figure, the volatility of the velocity $v_i$ is much smoother than that using the velocity obstacle method. Note that the velocity limitation is set as [-10,~10]~m/s.
\begin{figure}[t]
	\centering
	\includegraphics[width=0.4\textwidth, trim=10 20 0 10,clip]{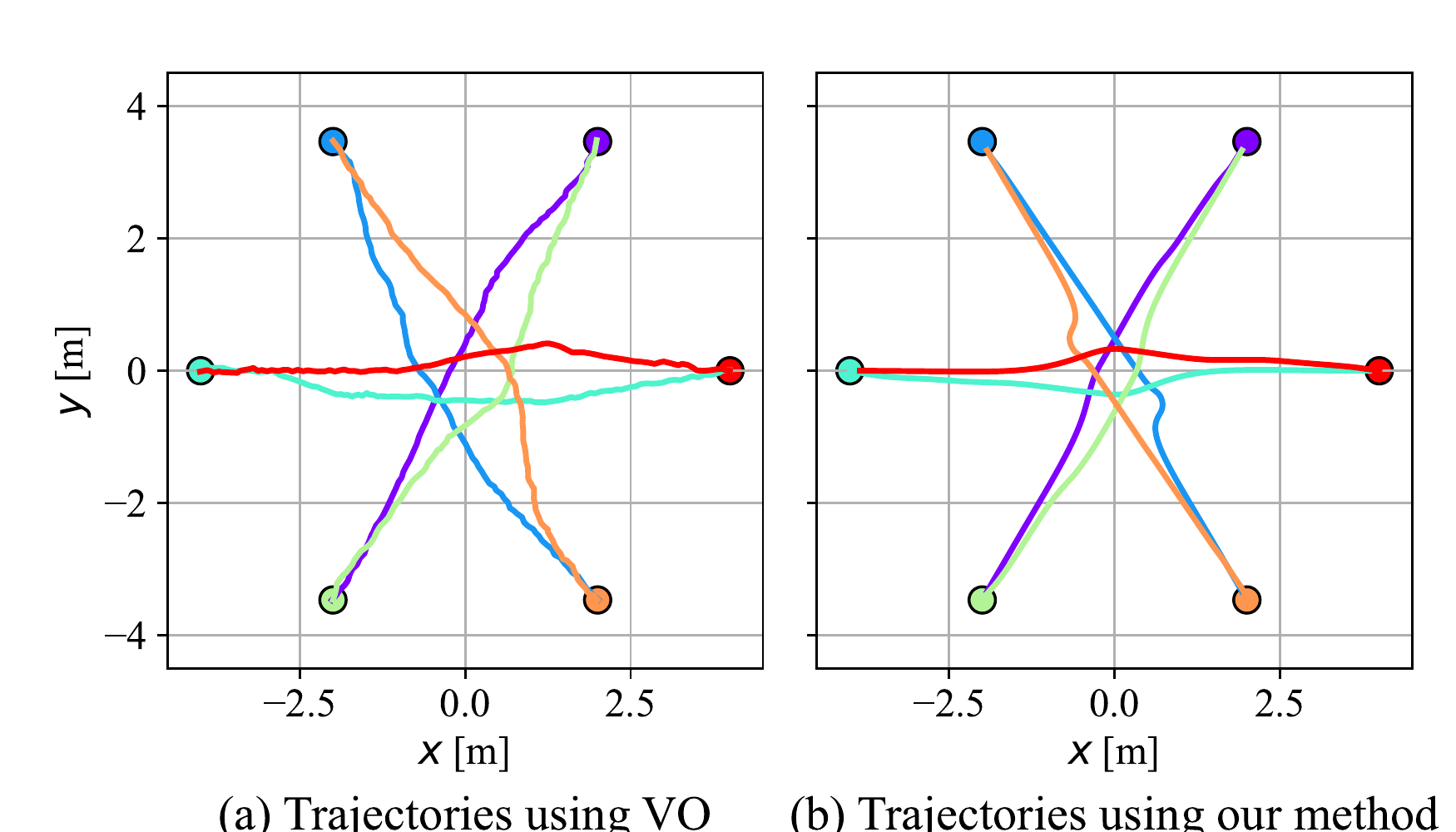}
	\caption{Resulting trajectories of 4 agents using VO and our method.}
	\label{fig:comp_traj_vo}
\end{figure}
\begin{figure}[t]
	\centering
	\includegraphics[width=0.4\textwidth, trim=0 20 0 15,clip]{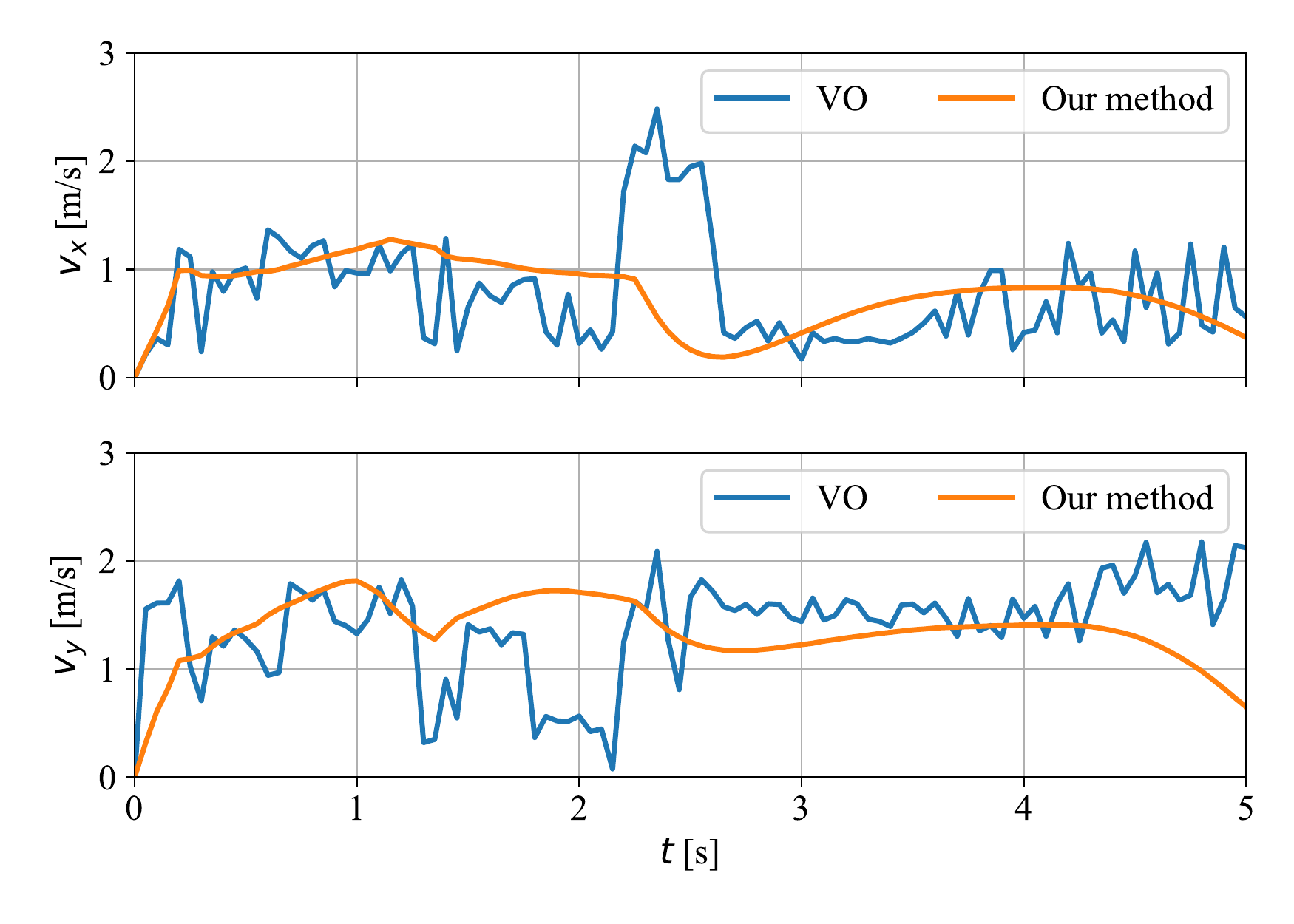}
	\caption{Resulting velocities of one agent using VO and our method (the agent is chosen randomly).}
	\label{fig:comp_velo_vo}
\end{figure}
In our proposed method, we set the control inputs, i.e., acceleration in $x$- and $y$-dimension, should be confined into the range [-10,~10]~m/s$^2$. However, as for the VO method, the control inputs are the command velocities, which means the VO method is just suitable for agents with no dynamics or a single-integrator dynamics model, instead of double-integrator dynamics. Thus, our proposed method also outperforms the VO method by considering the model of complex dynamics. The resulting control inputs also remain smooth since there are cost terms regarding the penalization of control inputs in our distributed motion planning method.  

Fig.~\ref{fig:comp_time_vo} is used to show the computational efficiency of our proposed method, compared with the VO method. This figure demonstrates the computing time of our proposed method and velocity obstacle method in each running step. Unfortunately, the VO method could generate a feasible and collision trajectory in a shorter time compared with our method, even though the trajectory of the VO method is highly nonsmooth.
\begin{figure}[th]
	\centering
	\includegraphics[width=0.4\textwidth, trim=0 20 0 10,clip]{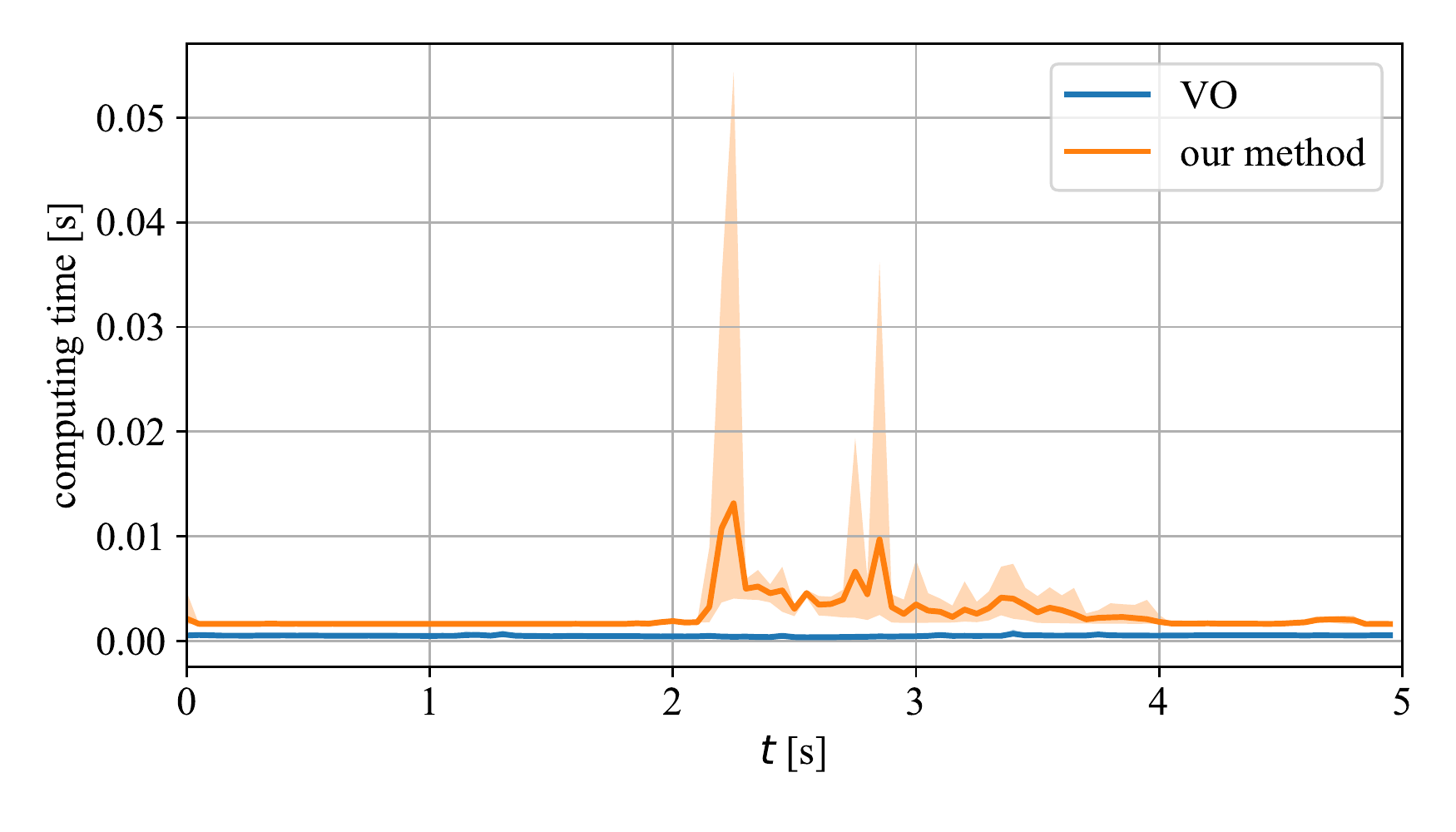}
	\caption{Changes of the computing time per agent using VO and our method.}
	\label{fig:comp_time_vo}
\end{figure}

The relationship between the computing time and the number of agents is demonstrated in Fig.~\ref{fig:comp_time_with_diff_number_vo}. For each number of agents, there are 20 random trials using the two methods; thus, the average computing time can be derived over all agents and all executions. Based on this figure, we can observe that the average computing time for the two methods increases with the number of agents increasing. Nevertheless, the increasing rate is different between the two methods. The average computational time of our proposed method rises faster than that of the velocity obstacle method. Thus, our proposed method will be more time-consuming compared with the VO method; nevertheless, the quality of resulting trajectories planned by our method is much higher than that of the VO method.
\begin{figure}[th]
	\centering
	\includegraphics[width=0.4\textwidth, trim=0 20 0 15, clip]{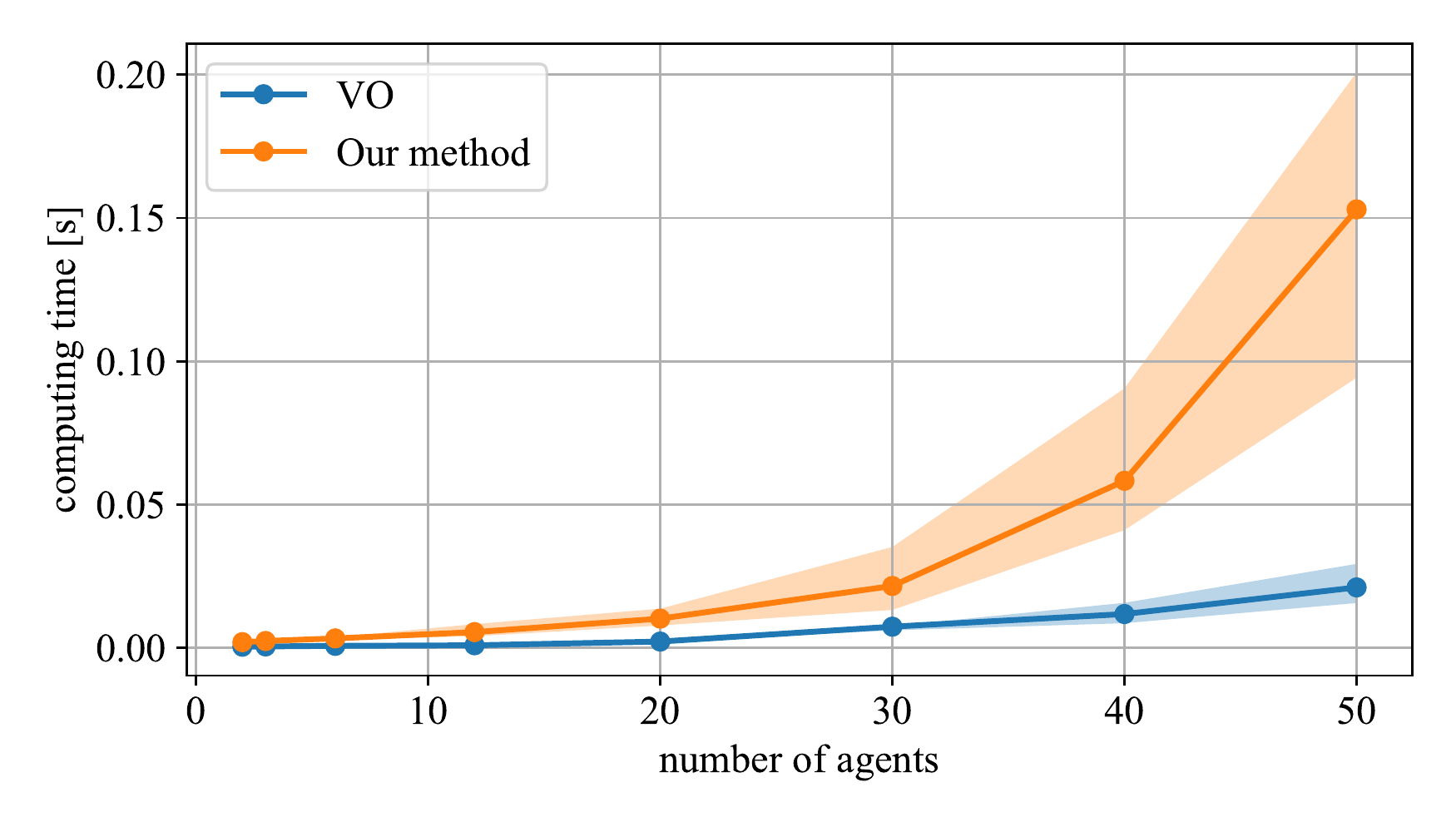}
	\caption{Changes of the computing time per agent using VO and our method with different numbers of agents.}
	\label{fig:comp_time_with_diff_number_vo}
\end{figure}

\subsection{Comparison with Position-Based MPC Method}
In this section, we show the comparison results of our proposed method and the MPC method that achieve the motion planning in position space. Different from our proposed motion planning method that takes the probabilistic chance constraints of velocity obstacles to avoid a collision in velocity space, the MPC method takes $\|p_i-p_j\|\geq d_{\mathrm{safe}}$ as collision-avoidance constraints to achieve the motion planning in position space. This MPC method is a well-applied collision-avoidance method and has been presented in many research works~\cite{dai2017distributed,huang2021dynamic}. In the following text, we abbreviate the MPC method in position space as PB-MPC for conciseness. All parameters in the PB-MPC method are set as the same value as in our method, and the $d_{\mathrm{safe}}$ in the PB-MPC method is set as $2r_i=0.4$~m.

The simulation task in this section is similar to the task in Section~\ref{section:results_compare1}. The number of agents is set as 6. The risk bound is $\delta_i=0.1$.
The trajectories of the 6 agents generated by the PB-MPC and our method are illustrated in Fig.~\ref{fig:comp_traj_mpc}. The initial positions of the 6 agents are represented by 6 circles with different colors. The radius of these circles in Fig.~\ref{fig:comp_traj_mpc} is equal to the agent radius $r_i=0.2$~m. According to this figure, both methods can navigate agents to reach their desired destinations without collision. However, it is straightforward that the resulting trajectories of our proposed method are much smoother than the trajectories of the PB-MPC method.  
\begin{figure}[t]
	\centering
	\includegraphics[width=0.4\textwidth, trim=10 20 0 20,clip]{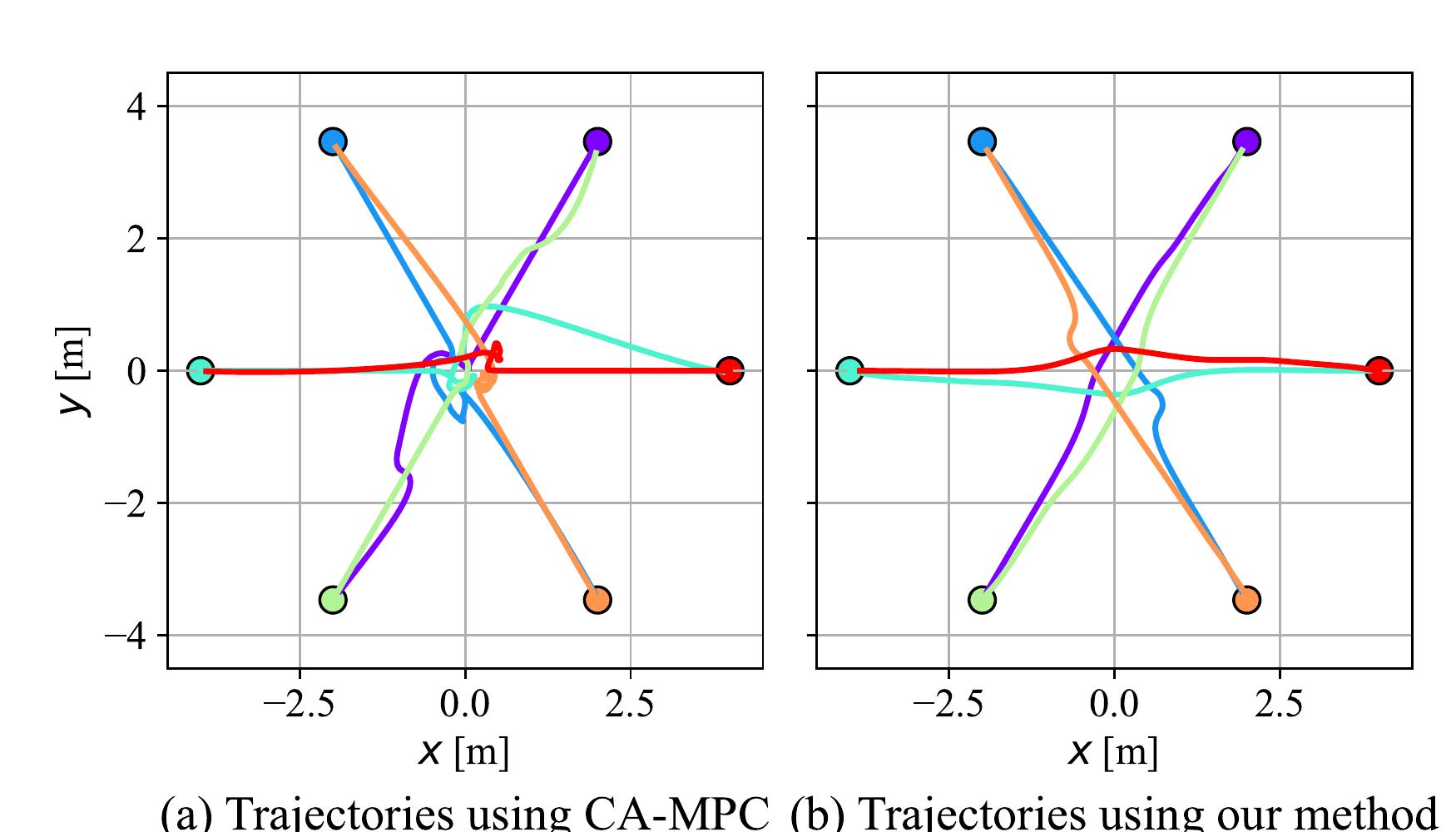}
	\caption{Trajectories of the 6 agents using the two methods.}
	\label{fig:comp_traj_mpc}
\end{figure}
Fig.~\ref{fig:comp_velo_mpc} demonstrates the change of velocities of one agent (randomly chosen from the 4 agents) with running step $t$ changing. The fluctuation around time interval [2,~4]~s indicates that the agent changes its velocity to avoid a collision with its neighbors. According to Fig.~\ref{fig:comp_velo_mpc}, the volatility of the velocity $v_i$ is much smoother than that of the PB-MPC method. 
The change of control inputs with the running step of the two methods is illustrated in Fig.~\eqref{fig:comp_input_mpc}. Here, the control inputs should be confined into the desired range [-10,~10]~m/s$^2$. Obviously, our proposed method could derive a smoother sequence of control inputs compared with the PB-MPC method.
\begin{figure}[th]
	\centering
	\includegraphics[width=0.4\textwidth, trim=0 20 0 15,clip]{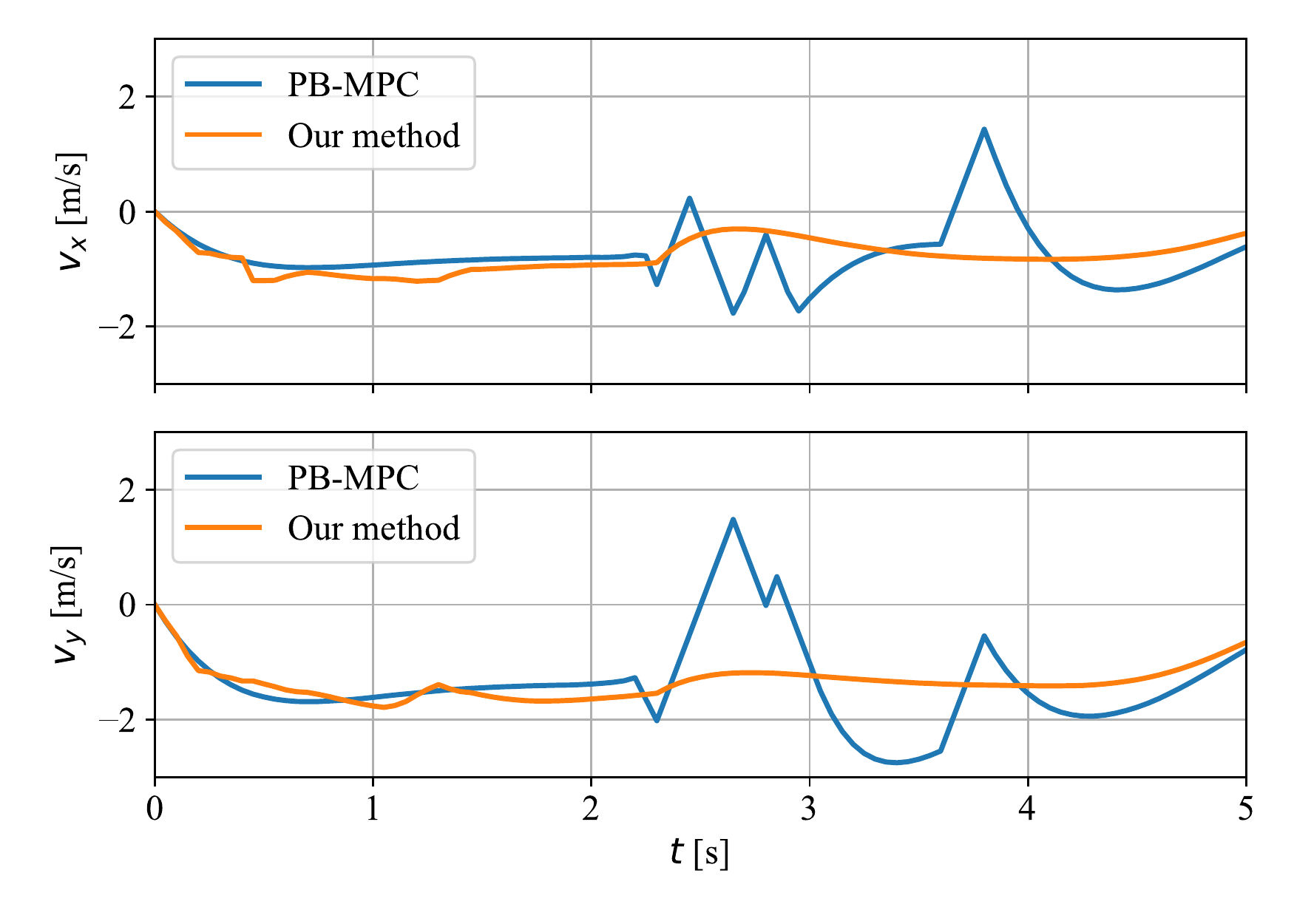}
	\caption{Velocities of one of the agents using the two methods (the agent is chosen randomly).}
	\label{fig:comp_velo_mpc}
\end{figure}
\begin{figure}[th]
	\centering
	\includegraphics[width=0.4\textwidth, trim=0 20 0 15,clip]{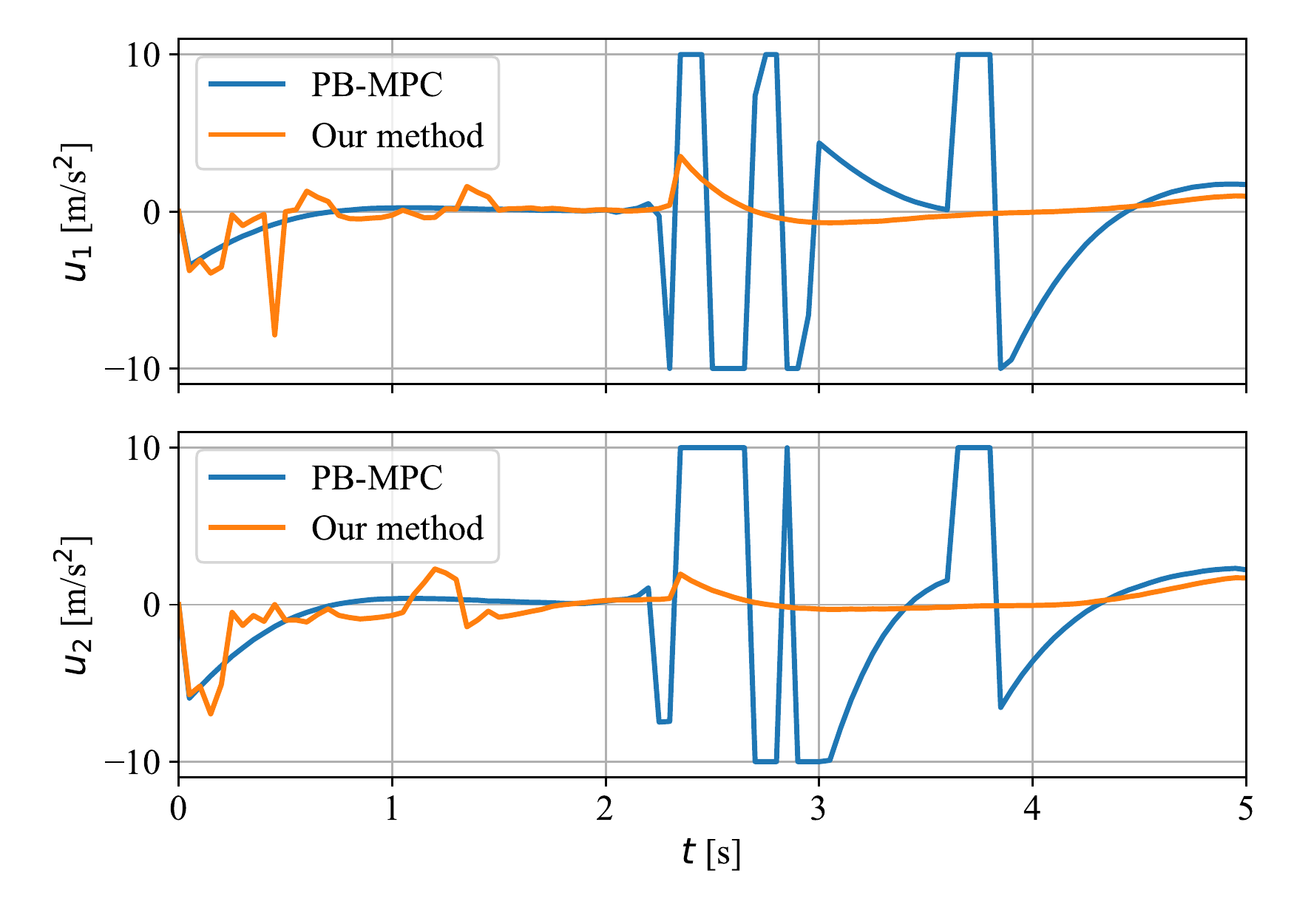}
	\caption{Control inputs of one of the agents using the two methods (the agent is chosen randomly).}
	\label{fig:comp_input_mpc}
\end{figure}

Fig.~\ref{fig:comp_time_mpc} is used to show the computational efficiency of our proposed method, compared with the PB-MPC method. This figure demonstrates the comparison of the average computing time per running step of the two methods. 
In this figure, the solid lines denote the average computing time of the two methods per running step per agent. The filled regions represent the range of computing time for all agents of the two methods. According to Fig.~\ref{fig:comp_time_mpc}, we can observe that the average computing time of our proposed method is always smaller than that of the PB-MPC method. 
Fig.~\ref{fig:comp_time_with_diff_number_mpc} presents the change of average computing time per agent with respect to different number of agents. Obviously, the average computing time increases with the number of agents increasing for both methods. Nevertheless, the increasing rate is different between the two methods. The average computational time of our proposed method increases slower than that of the PB-MPC method. 
\begin{figure}[th]
	\centering
	\includegraphics[width=0.4\textwidth, trim=0 20 0 15,clip]{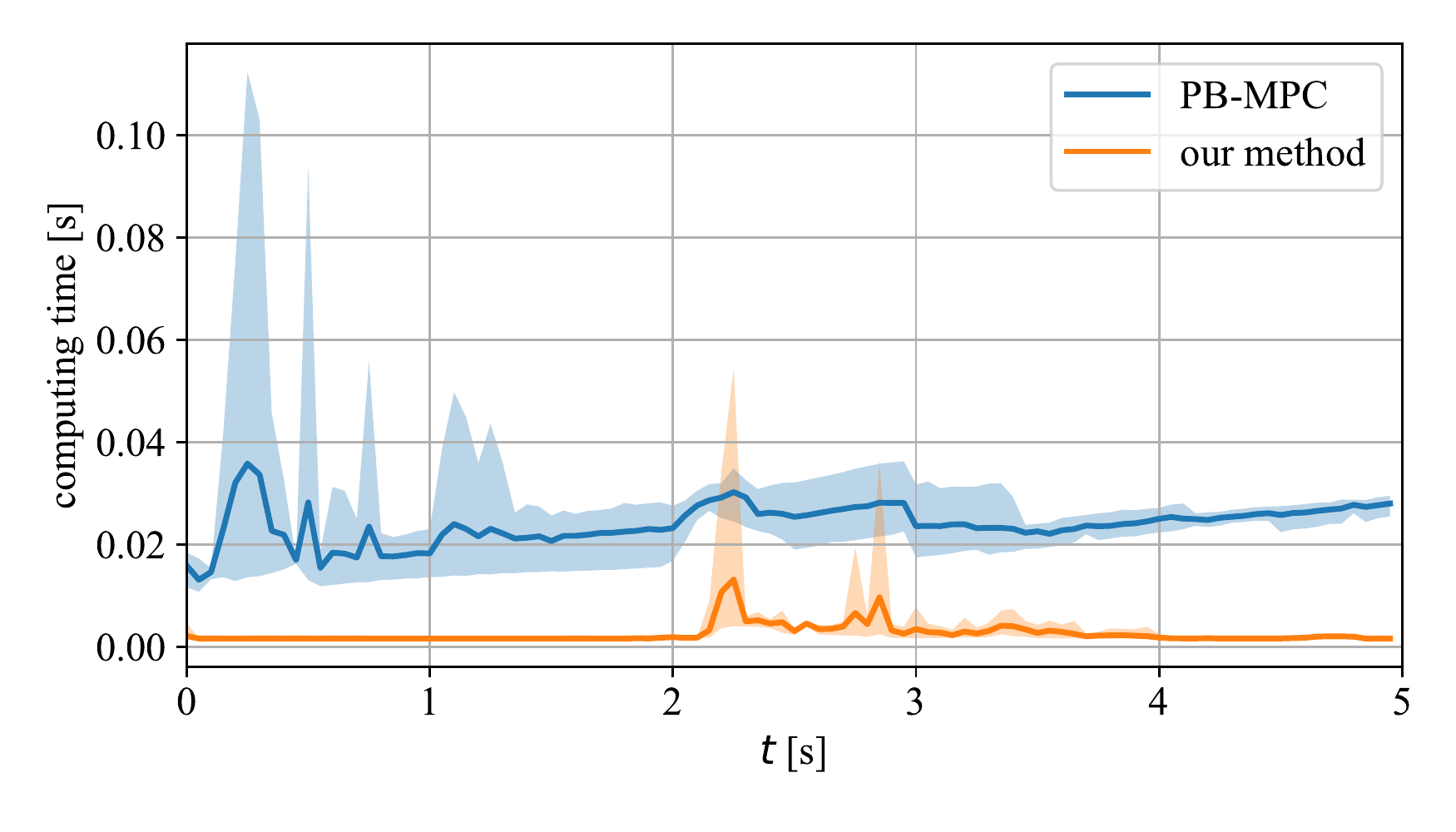}
	\caption{Average computing time using the two methods at each running step.}
	\label{fig:comp_time_mpc}
\end{figure}
\begin{figure}[thp]
	\centering
	\includegraphics[width=0.4\textwidth, trim=0 20 0 15, clip]{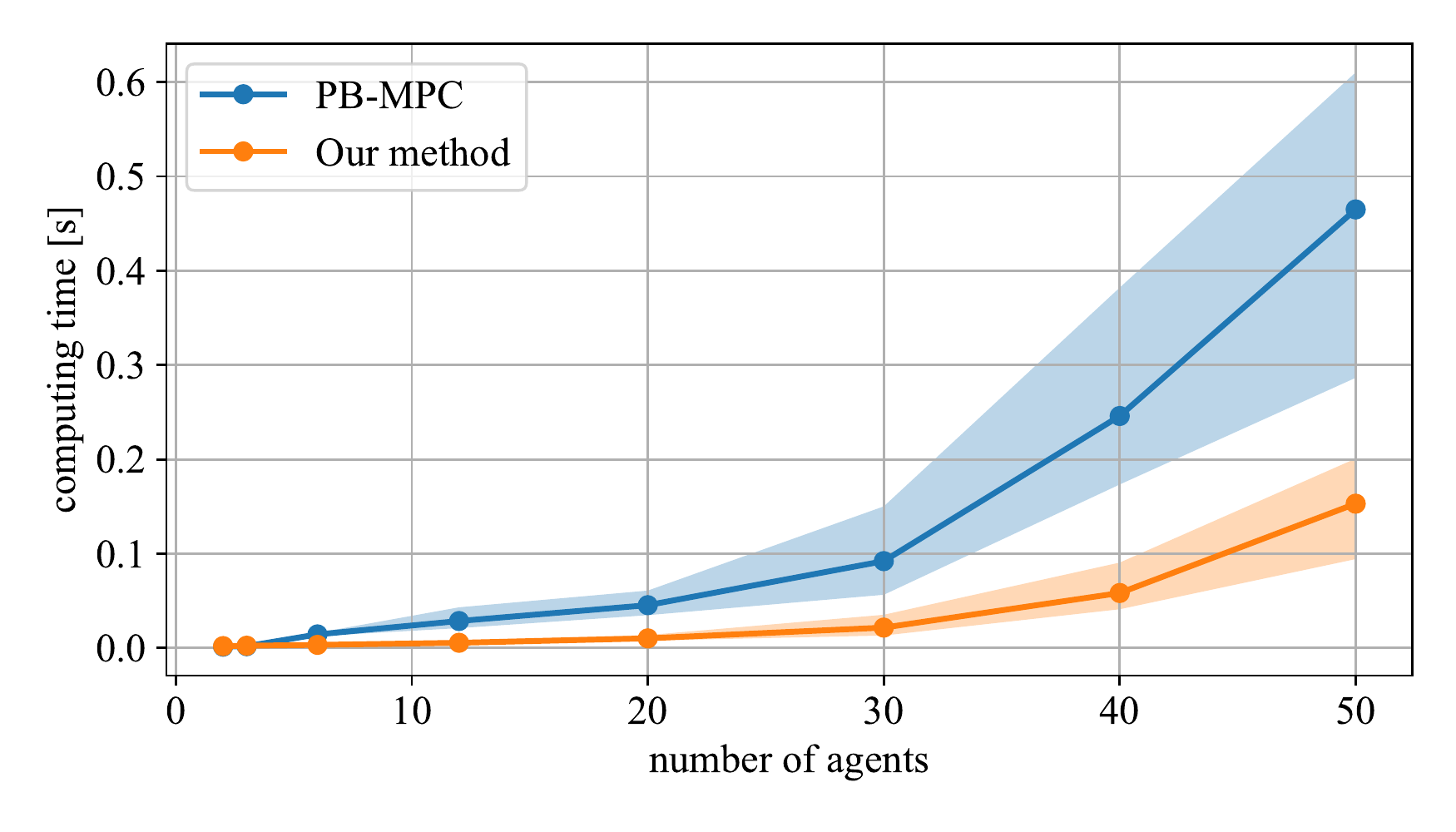}
	\caption{Average computing time using the two methods regarding different number of agents.}
	\label{fig:comp_time_with_diff_number_mpc}
\end{figure}

In addition, we compare the results of our method under three different levels of measurement noise $\frac{1}{4}W_i$, $W_i$, and $4W_i$ with the PB-MPC method. The minimum distance between each pair of agents and the success rate are treated as the safety metrics. Since the radius of agents is 0.2~m, the minimum safety distance without collision should be $2r_i = 0.4$~m. The comparison results are shown in Table~\ref{tab:diff_noise}. Besides, due to a tighter bound of collision risk approximation, our method can keep a larger minimum distance while running under the same noise level, compared with the PB-MPC method. Also, with a larger noise level, our proposed method maintains the success rate of 100\%, but the success rate of PB-MPC decreases from 91\% to 78\%. According to Table~\ref{tab:diff_noise}, we can observe that our proposed method achieves higher safety performance compared with the PB-MPC method.
\begin{table}[th]	
	\renewcommand{\arraystretch}{1.3}
	\caption{Trajectory Safety Comparison of Two Algorithms with Different Levels of Noise (The Values Are Computed from Successful Runs).}
	\label{tab:diff_noise}
	\centering
	\begin{tabular}{|c|c|c|c|}
		\hline
		Noise &  Safety metrics   & PB-MPC method & Our method  \\ \hline
		\multirow{2}{*}{$\frac{1}{4}W_i$} & Minimum distance  & 0.412 m   & 0.447 m  \\ \cline{2-4} 
		& Success rate       & 91 \%   & 100 \%                        \\ \hline
		\multirow{2}{*}{$W_i$}            & Minimum distance    & 0.436 m    & 0.469 m   \\ \cline{2-4} 
		& Success rate        & 84 \%     & 100 \%                     \\ \hline
		\multirow{2}{*}{$4W_i$}            & Minimum distance    & 0.448 m   & 0.491 m    \\ \cline{2-4} 
		& Success rate         & 78\%     & 100 \%                   \\ \hline
	\end{tabular}
\end{table}

\section{Conclusion}
\label{section:conclusion}
In this paper, a velocity obstacle based risk-bounded robust motion planning method is proposed for stochastic multi-agent systems. In this method, the disturbance, noise, and model uncertainty are considered to enhance the robustness of this method. A chance-constrained MPC problem is formulated based on the feasible region of velocity vector provided by the velocity obstacles method. The feasible collision-free regions of the ego agent's velocity vector and position vector are derived and formulated as probabilistic collision constraints. Hence the proposed method plans the trajectories at the velocity space to avoid collisions with other agents or moving obstacles; and thus, the quality of resulting trajectories and computational efficiency is certainly improved. The introduction of chance constraints also guarantees the appropriate bound of potential collision risk during the robust motion planning process for the stochastic system. Several simulation scenarios for multiple agents are employed to validate the effectiveness and efficiency of our proposed methodology.


\section*{APPENDIX}
\label{appendix:param}
Table~\ref{tab:setting_params} shows the value of all parameters in simulation.
\begin{table}[th]
	\renewcommand{\arraystretch}{1.3}
	\caption{Parameter Setting}
	\label{tab:setting_params}
	\begin{center}
		\begin{tabular}{|p{.12\textwidth}<{\centering}|p{.045\textwidth}<{\centering}|p{.18\textwidth}<{\centering}|p{.045\textwidth}<{\centering}| }
			\hline
			Meaning & Notation & Value &Unit \\  \hline
			Radius & $r_i$ & 0.2 & m \\ \hline
			Sampling time & $\tau_s$ & 0.05 & s \\ \hline
			Running steps & $t_{\textup{run}}$ & 5 & s \\ \hline
			State weighting matrix & $Q_i$  $Q_{i|N}$ & $\operatorname{diag}(10, 10, 0.1, 0.1)$ $\operatorname{diag}(10, 10, 0,0)$ & -  \\ \hline
			Weighting matrix & $R_i$ & $\operatorname{diag}(0.1, 0.1)$ & -  \\ \hline
			Minimum states & $x_{\min,i}$ & $\left[-\infty \; -\infty \; -10 \; -10\right]^\top$   & m, m/s  \\ \hline
			Maximum states & $x_{\max,i}$ & $\left[ \infty \;  \infty \; 10 \; 10\right]^\top$  & m, m/s  \\ \hline
			Minimum inputs & $u_{\min,i}$ & $\left[ -10 \;  -10\right]^\top$ &  {m/s}$^2$  \\ \hline 
			Maximum inputs & $u_{\max,i}$ & $\left[ 10 \; 10 \right]^\top$   &  {m/s}$^2$  \\ \hline
			Noise covariance & $W_i$ & $\operatorname{diag}$(1e-4,1e-4,1e-2,1e-2) & - \\ \hline
			Noise covariance & $P_i$ & $\operatorname{diag}$(1e-6,1e-6,1e-6,1e-6) & - \\ \hline
			Constant scalar & $M$ & 1e4 & - \\ \hline
		\end{tabular}
	\end{center}
\end{table}

\bibliographystyle{IEEEtran}
\bibliography{IEEEabrv,Reference}

\begin{thebibliography}{10}
\providecommand{\url}[1]{#1}
\csname url@samestyle\endcsname
\providecommand{\newblock}{\relax}
\providecommand{\bibinfo}[2]{#2}
\providecommand{\BIBentrySTDinterwordspacing}{\spaceskip=0pt\relax}
\providecommand{\BIBentryALTinterwordstretchfactor}{4}
\providecommand{\BIBentryALTinterwordspacing}{\spaceskip=\fontdimen2\font plus
\BIBentryALTinterwordstretchfactor\fontdimen3\font minus
  \fontdimen4\font\relax}
\providecommand{\BIBforeignlanguage}[2]{{%
\expandafter\ifx\csname l@#1\endcsname\relax
\typeout{** WARNING: IEEEtran.bst: No hyphenation pattern has been}%
\typeout{** loaded for the language `#1'. Using the pattern for}%
\typeout{** the default language instead.}%
\else
\language=\csname l@#1\endcsname
\fi
#2}}
\providecommand{\BIBdecl}{\relax}
\BIBdecl

\bibitem{semnani2020force}
S.~H. Semnani, A.~H. de~Ruiter, and H.~H. Liu, ``Force-based algorithm for
  motion planning of large agent,'' \emph{IEEE Transactions on Cybernetics},
  vol.~52, no.~1, pp. 654--665, 2022.

\bibitem{hang2022decision}
P.~Hang, C.~Huang, Z.~Hu, and C.~Lv, ``Decision making for connected automated
  vehicles at urban intersections considering social and individual benefits,''
  \emph{arXiv preprint arXiv:2201.01428}, 2022.

\bibitem{zhang2020trajectory}
X.~Zhang, J.~Ma, Z.~Cheng, S.~Huang, S.~S. Ge, and T.~H. Lee, ``Trajectory
  generation by chance-constrained nonlinear mpc with probabilistic
  prediction,'' \emph{IEEE Transactions on Cybernetics}, vol.~51, no.~7, pp.
  3616--3629, 2021.

\bibitem{hang2021cooperative}
P.~Hang, C.~Lv, C.~Huang, Y.~Xing, and Z.~Hu, ``Cooperative decision making of
  connected automated vehicles at multi-lane merging zone: A coalitional game
  approach,'' \emph{IEEE Transactions on Intelligent Transportation Systems},
  2021.

\bibitem{lyu2019multivehicle}
Y.~Lyu, J.~Hu, B.~M. Chen, C.~Zhao, and Q.~Pan, ``Multivehicle flocking with
  collision avoidance via distributed model predictive control,'' \emph{IEEE
  Transactions on Cybernetics}, vol.~51, no.~5, pp. 2651--2662, 2021.

\bibitem{zhang2021sequential}
X.~Zhang, J.~Ma, Z.~Cheng, F.~L. Lewis, and T.~H. Lee, ``Sequential convex
  programming for collaboration of connected and automated vehicles,''
  \emph{IEEE Transactions on Intelligent Vehicles}, 2021.

\bibitem{duan2021adaptive}
J.~Duan, Z.~Liu, S.~E. Li, Q.~Sun, Z.~Jia, and B.~Cheng, ``Adaptive dynamic
  programming for nonaffine nonlinear optimal control problem with state
  constraints,'' \emph{Neurocomputing}, 2021.

\bibitem{pan2021multilayer}
Z.~Pan, Z.~Sun, H.~Deng, and D.~Li, ``A multilayer graph for multiagent
  formation and trajectory tracking control based on {MPC} algorithm,''
  \emph{IEEE Transactions on Cybernetics}, 2021.

\bibitem{bono2021swarm}
A.~Bono, G.~Fedele, and G.~Franz{\`e}, ``A swarm-based distributed model
  predictive control scheme for autonomous vehicle formations in uncertain
  environments,'' \emph{IEEE Transactions on Cybernetics}, 2021.

\bibitem{fiorini1998motion}
P.~Fiorini and Z.~Shiller, ``Motion planning in dynamic environments using
  velocity obstacles,'' \emph{The International Journal of Robotics Research},
  vol.~17, no.~7, pp. 760--772, 1998.

\bibitem{douthwaite2019velocity}
J.~A. Douthwaite, S.~Zhao, and L.~S. Mihaylova, ``Velocity obstacle approaches
  for multi-agent collision avoidance,'' \emph{Unmanned Systems}, vol.~7,
  no.~01, pp. 55--64, 2019.

\bibitem{van2011reciprocal}
J.~Van Den~Berg, S.~J. Guy, M.~Lin, and D.~Manocha, ``Reciprocal n-body
  collision avoidance,'' in \emph{Robotics Research}.\hskip 1em plus 0.5em
  minus 0.4em\relax Springer, 2011, pp. 3--19.

\bibitem{wilkie2009generalized}
D.~Wilkie, J.~Van Den~Berg, and D.~Manocha, ``Generalized velocity obstacles,''
  in \emph{Proceedings of IEEE/RSJ International Conference on Intelligent
  Robots and Systems}.\hskip 1em plus 0.5em minus 0.4em\relax IEEE, 2009, pp.
  5573--5578.

\bibitem{snape2011hybrid}
J.~Snape, J.~Van Den~Berg, S.~J. Guy, and D.~Manocha, ``The hybrid reciprocal
  velocity obstacle,'' \emph{IEEE Transactions on Robotics}, vol.~27, no.~4,
  pp. 696--706, 2011.

\bibitem{zhang2021semi}
X.~Zhang, Z.~Cheng, J.~Ma, S.~Huang, F.~L. Lewis, and T.~H. Lee,
  ``Semi-definite relaxation based admm for cooperative planning and control of
  connected autonomous vehicles,'' \emph{IEEE Transanctions on Intelligent
  Transportation Systems}, 2021.

\bibitem{zhang2019integrated}
X.~Zhang, J.~Ma, S.~Huang, Z.~Cheng, and T.~H. Lee, ``Integrated planning and
  control for collision-free trajectory generation in {3D} environment with
  obstacles,'' in \emph{Proceedings of International Conference on Control,
  Automation and Systems}.\hskip 1em plus 0.5em minus 0.4em\relax IEEE, 2019,
  pp. 974--979.

\bibitem{huang2021dynamic}
C.~Huang, R.~Liu, and J.~Xu, ``Dynamic motion planning for driverless vehicles
  via decentralized model predictive control,'' in \emph{Proceedings of
  International Conference on Mechatronics and Machine Vision in
  Practice}.\hskip 1em plus 0.5em minus 0.4em\relax IEEE, 2021, pp. 594--599.

\bibitem{katriniok2019nonlinear}
A.~Katriniok, P.~Sopasakis, M.~Schuurmans, and P.~Patrinos, ``Nonlinear model
  predictive control for distributed motion planning in road intersections
  using {PANOC},'' in \emph{Proceedings of IEEE 58th Conference on Decision and
  Control}.\hskip 1em plus 0.5em minus 0.4em\relax IEEE, 2019, pp. 5272--5278.

\bibitem{hegde2016multi}
R.~Hegde and D.~Panagou, ``Multi-agent motion planning and coordination in
  polygonal environments using vector fields and model predictive control,'' in
  \emph{Proceedings of European Control Conference}.\hskip 1em plus 0.5em minus
  0.4em\relax IEEE, 2016, pp. 1856--1861.

\bibitem{zhu2019chance}
H.~Zhu and J.~Alonso-Mora, ``Chance-constrained collision avoidance for {MAV}s
  in dynamic environments,'' \emph{IEEE Robotics and Automation Letters},
  vol.~4, no.~2, pp. 776--783, 2019.

\bibitem{peng2021seperated}
B.~Peng, Y.~Mu, J.~Duan, Y.~Guan, S.~E. Li, and J.~Chen, ``Separated
  proportional-integral lagrangian for chance constrained reinforcement
  learning,'' in \emph{Proceedings of IEEE Intelligent Vehicles
  Symposium}.\hskip 1em plus 0.5em minus 0.4em\relax IEEE, 2021, pp. 193--199.

\bibitem{blackmore2011chance}
L.~Blackmore, M.~Ono, and B.~C. Williams, ``Chance-constrained optimal path
  planning with obstacles,'' \emph{IEEE Transactions on Robotics}, vol.~27,
  no.~6, pp. 1080--1094, 2011.

\bibitem{sherali2012optimization}
H.~D. Sherali and C.~M. Shetty, \emph{Optimization with disjunctive
  constraints}.\hskip 1em plus 0.5em minus 0.4em\relax Springer Science \&
  Business Media, 2012, vol. 181.

\bibitem{floudas2005mixed}
C.~A. Floudas and X.~Lin, ``Mixed integer linear programming in process
  scheduling: Modeling, algorithms, and applications,'' \emph{Annals of
  Operations Research}, vol. 139, no.~1, pp. 131--162, 2005.

\bibitem{dai2017distributed}
L.~Dai, Q.~Cao, Y.~Xia, and Y.~Gao, ``Distributed {MPC} for formation of
  multi-agent systems with collision avoidance and obstacle avoidance,''
  \emph{Journal of the Franklin Institute}, vol. 354, no.~4, pp. 2068--2085,
  2017.

\end{thebibliography}

\end{document}